\documentclass{article}

\usepackage{microtype}
\usepackage{graphicx}
\usepackage{subcaption}
\usepackage{booktabs} % for professional tables

\usepackage{hyperref}

\usepackage[preprint]{icml2026}

\usepackage{amsmath}
\usepackage{amssymb}
\usepackage{mathtools}
\usepackage{amsthm}
\usepackage{enumitem}
\usepackage{array}
\usepackage{tabularx}

\usepackage[capitalize,noabbrev]{cleveref}

\theoremstyle{plain}
\newtheorem{theorem}{Theorem}[section]

\newtheorem{lemma}[theorem]{Lemma}

\theoremstyle{definition}
\newtheorem{definition}[theorem]{Definition}
\newtheorem{assumption}[theorem]{Assumption}
\theoremstyle{remark}

\usepackage[textsize=tiny]{todonotes}

\usepackage{xurl} % allow line breaks in long URLs
\usepackage{floatflt}
\usepackage{makecell}
\usepackage{multirow}
\usepackage[table]{xcolor} % 用于表格颜色设置（可选）
\usepackage{pifont}
\usepackage{tikz}
\usepackage{colortbl} % 支持表格底色
\usepackage{dblfloatfix}
\definecolor{lightpurple}{RGB}{230,230,250}

\icmltitlerunning{REPAIR: Robust Editing via Progressive Adaptive Intervention and Reintegration}

\begin{document}

\twocolumn[
  \icmltitle{REPAIR: Robust Editing via Progressive Adaptive Intervention and Reintegration}

  % It is OKAY to include author information, even for blind submissions: the
  % style file will automatically remove it for you unless you've provided
  % the [accepted] option to the icml2026 package.

  % List of affiliations: The first argument should be a (short) identifier you
  % will use later to specify author affiliations Academic affiliations
  % should list Department, University, City, Region, Country Industry
  % affiliations should list Company, City, Region, Country

  % You can specify symbols, otherwise they are numbered in order. Ideally, you
  % should not use this facility. Affiliations will be numbered in order of
  % appearance and this is the preferred way.
  \icmlsetsymbol{equal}{*}

  \begin{icmlauthorlist}
    \icmlauthor{Yisu Wang}{equal,tecsu,ucas}
    \icmlauthor{Ming Wang}{equal,neu}
    \icmlauthor{Haoyuan Song}{nyu}
    \icmlauthor{Wenjie Huang}{nus}
    \icmlauthor{Chaozheng Wang}{cuhk}
    \icmlauthor{Yi Xie}{tum}
    %\icmlauthor{}{sch}
    \icmlauthor{Xuming Ran}{nus}
    %\icmlauthor{}{sch}
    %\icmlauthor{}{sch}
  \end{icmlauthorlist}

  \icmlaffiliation{tecsu}{Technological and Engineering Center for Space Utilization of Chinese Academy of Sciences, Beijing, China}
  \icmlaffiliation{ucas}{University of Chinese Academy of Sciences, Beijing, China}
  \icmlaffiliation{neu}{School of Computer Science and Engineering, Northeastern University, Shenyang, China}
  \icmlaffiliation{nyu}{Courant Institute School of Mathematics, Computing, and Data Science, New York University, New York, USA}
  \icmlaffiliation{nus}{School of Computing, National University of Singapore, Singapore}
  \icmlaffiliation{cuhk}{The Chinese University of Hong Kong, Hong Kong, China}
  \icmlaffiliation{tum}{Technical University of Munich, Baden-Württemberg, Germany}

  \icmlcorrespondingauthor{Xuming Ran}{ranxuming@gmail.com}

  % You may provide any keywords that you find helpful for describing your
  % paper; these are used to populate the "keywords" metadata in the PDF but
  % will not be shown in the document
  \icmlkeywords{Machine Learning, ICML}

  \vskip 0.3in
]

% this must go after the closing bracket ] following \twocolumn[ ...

% This command actually creates the footnote in the first column listing the
% affiliations and the copyright notice. The command takes one argument, which
% is text to display at the start of the footnote. The \icmlEqualContribution
% command is standard text for equal contribution. Remove it (just {}) if you
% do not need this facility.

% Use ONE of the following lines. DO NOT remove the command.
% If you have no special notice, KEEP empty braces:
\printAffiliationsAndNotice{}  % no special notice (required even if empty)
% Or, if applicable, use the standard equal contribution text:
% \printAffiliationsAndNotice{\icmlEqualContribution}

\begin{abstract}
  Post-training of large language models (LLMs) is constrained by the high cost of acquiring new knowledge or correcting errors, as well as the unintended side effects that frequently arise from retraining. To address these issues, we introduce REPAIR (\textbf{R}obust \textbf{E}diting via \textbf{P}rogressive \textbf{A}daptive \textbf{I}ntervention and \textbf{R}eintegration), a lifelong editing framework designed to support precise and low-cost model updates while preserving non-target knowledge.
  % REPAIR is engineered to overcome the key hurdles in model editing.
  REPAIR mitigates the instability and conflicts of large-scale sequential edits through a closed-loop feedback mechanism coupled with dynamic memory management. Furthermore, by incorporating frequent knowledge fusion and enforcing strong locality guards, REPAIR effectively addresses the shortcomings of traditional distribution-agnostic approaches that often overlook unintended ripple effects.
  Our experiments demonstrate that REPAIR boosts editing accuracy by 10\%--30\% across multiple model families and significantly reduces knowledge forgetting. This work introduces a robust framework for developing reliable, scalable, and continually evolving LLMs. The code is available at \url{https://github.com/sci-m-wang/REPAIR}.
\end{abstract}

\section{Introduction}
\begin{figure}[t]
  \centering
  \includegraphics[width=\columnwidth]{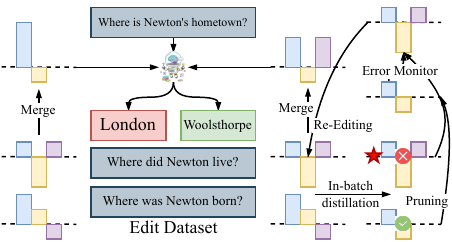}
  \caption{\textbf{Problems and our solutions}. REPAIR achieves closed-loop feedback, fine-grained knowledge integration, weighted knowledge merging, and robust editing performance.}
  \label{fig:wrap_example}
\end{figure}
Large language models (LLMs) have demonstrated remarkable capabilities across diverse tasks. However, their inherent rigidity prevents them from autonomously updating knowledge after pre-training, rendering them unable to correct errors (\emph{e.g.}, hallucinations or outdated facts) or integrate new information. Consequently, lifelong model editing has emerged as a critical research paradigm. It aims to enable continuous, efficient, and low-cost local updates that ensure models remain accurate and relevant over time \cite{Wang2024KESurvey}. In contrast to full retraining or broad fine-tuning, editing focuses on \emph{precisely fine-grained} modifications that preserve unrelated competencies while delivering immediate corrections at deployment time.

Despite steady progress, important gaps remain, as shown in Figure~\ref{fig:wrap_example}. \textbf{(1) Large-scale sequential editing \& coarse knowledge fusion}. As edits accumulate, models can exhibit routing instability, conflicts among edits, and even collapse. Thus, stabilizing sequential updates without broad side effects remains challenging \cite{Gupta2024RebuildROME,Cohen2024Ripple}. Semi-parametric designs (\emph{e.g.}, SERAC \cite{Mitchell2022SERAC}) and discrete key–value adaptors (\emph{e.g.}, GRACE \cite{Hartvigsen2023GRACE}) alleviate some failure modes and support long edit streams, but still face scope and auditing trade-offs \cite{Mitchell2022SERAC,Hartvigsen2023GRACE}. The strategy for knowledge fusion remains underexplored, despite being the stage most prone to information loss \cite{Wang2024WISE}. \textbf{(2) Few-shot editing}. Under data-scarce conditions, editors often struggle to form robust, generalizable changes beyond the exact prompt, motivating gradient-transformation editors trained for locality (\emph{e.g.}, MEND \cite{Mitchell2022MEND}) and broader taxonomies of edit generalization \cite{Mitchell2022MEND,Wang2024KESurvey}. \textbf{(3) Open-loop and distribution-agnostic learning}. Many pipelines operate without reflective feedback, optimize on indiscriminate batches, and under-stress-test ripple effects on related knowledge and reasoning, calling for tighter evaluation and integration mechanisms \cite{Cohen2024Ripple,Wang2024KESurvey}. Overall, these issues reveal a fundamental trade-off among reliability, specificity, and scalability that any practical editing system must reconcile.

\newcommand{\cmark}{\textcolor{cyan}{\ding{51}}}   % ✔
\newcommand{\xmark}{\textcolor{red}{\ding{55}}}    % ✗

To address these challenges, we propose the framework named REPAIR (\textbf{R}obust \textbf{E}diting via \textbf{P}rogressive \textbf{A}daptive \textbf{I}ntervention and \textbf{R}eintegration), with targeted strategies: \textbf{(1) Closed-loop feedback with dynamic memory management} that monitors edit performance and selectively reinitializes underperforming modules to stabilize routing and consolidation at scale. Concretely, our controller triggers health checks after each edit window and performs scoped resets or compaction when drift is detected. \textbf{(2) Distribution-aware optimization} that reorganizes samples by similarity and applies inner-batch distillation to enhance consistency and robustness in few-shot settings, encouraging edits to generalize across paraphrases and nearby contexts rather than overfitting to single prompts. \textbf{(3) Frequent knowledge fusion} that increases fusion cadence to prevent information loss and ensure timely consolidation of new and existing knowledge, with guardrails that validate locality before integration to avoid unintended side effects.

We compare \textsc{REPAIR} with several foundational model editing methods across three dimensions: \textit{Memory}, \textit{Attributes}, and \textit{Behaviors} (Table \ref{tab:model_comparison}). Its core innovation lies in integrating a dual memory system with parametric editing, complemented by error feedback, inner-batch knowledge distillation, and loss-aware subspaces merging. This design achieves high success rates and broad editing coverage while minimizing side effects. In contrast, previous methods struggle with knowledge overlap and loss, particularly in sequential editing, where large differences between adjacent samples hinder effective correction. Table \ref{Tab:case study} showcases cases where \textsc{REPAIR} outperforms baselines, offering a better balance of Reliability, Generalization, and Locality.

\begin{table*}[h]\scriptsize
  \centering
  \setlength{\tabcolsep}{2pt}
  \caption{\textbf{Comparison of current model editing methods.} ``\checkmark'' refers to ``yes'' and ``well-supported'',
    ``$\times$'' refers to ``no'' or ``badly-supported'', and ``$\bigcirc$'' refers to ``less-supported''.
  The three metrics of Reliability, Generalization, and Locality denote the performance on lifelong editing.}
  \label{tab:model_comparison}
  \begin{tabular}{l|ccc|cccc|cc}
    \toprule
    & \multicolumn{3}{c|}{\textbf{Memory}} & \multicolumn{4}{c|}{\textbf{Attributes}} & \multicolumn{2}{c}{\textbf{Behaviors}} \\
    \cmidrule(lr){2-4}\cmidrule(lr){5-8}\cmidrule(l){9-10}
    Methods & \makecell{Long-term\\Memory} & \makecell{Working\\Memory} & Parametric & Lifelong & Reliability & Generalization & Locality & \makecell{Error\\Feedback} & \makecell{Knowledge\\Distillation} \\
    \midrule
    FT-EWC \cite{Kirkpatrick2017EWC} & $\checkmark$ & $\times$ & $\checkmark$ & $\checkmark$ & $\checkmark$ & $\checkmark$ & $\times$ & $\times$ & $\times$ \\
    ROME \cite{meng2022locating}       & $\checkmark$ & $\times$ & $\checkmark$ & $\times$ & $\times$     & $\times$     & $\times$ & $\times$ & $\times$ \\
    MEMIT \cite{Meng2023MEMIT}         & $\checkmark$ & $\times$ & $\checkmark$ & $\times$ & $\times$     & $\times$     & $\times$ & $\times$ & $\times$ \\
    MEND \cite{Mitchell2022MEND}       & $\checkmark$ & $\times$ & $\checkmark$ & $\times$ & $\times$     & $\times$     & $\times$ & $\times$ & $\times$ \\
    DEFER \cite{Mitchell2022SERAC}     & $\times$     & $\checkmark$ & $\checkmark$ & $\checkmark$ & $\bigcirc$ & $\times$ & $\times$ & $\times$ & $\times$ \\
    GRACE \cite{Hartvigsen2023GRACE}   & $\times$     & $\checkmark$ & $\times$ & $\checkmark$ & $\checkmark$     & $\times$ & $\checkmark$ & $\times$ & $\times$ \\
    WISE \cite{Wang2024WISE}           & $\checkmark$ & $\checkmark$ & $\checkmark$ & $\checkmark$ & $\checkmark$ & $\checkmark$ & $\checkmark$ & $\times$ & $\times$ \\
    \midrule
    REPAIR                              & $\checkmark$ & $\checkmark$ & $\checkmark$ & $\checkmark$ & $\checkmark$ & $\checkmark$ & $\checkmark$ & $\checkmark$ & $\checkmark$ \\
    \bottomrule
  \end{tabular}
\end{table*}

\begin{table}[t]
  \centering
  \caption{\textbf{Failure case study}. Previous baselines (\citep{Wang2024WISE}, \citep{Hartvigsen2023GRACE}) often encounter issues of repeating answers from previous questions and difficulty in correcting adjacent knowledge during editing.}
  \label{Tab:case study}
  \scriptsize
  \setlength{\tabcolsep}{2pt}
  \renewcommand{\arraystretch}{1.1}
  \begin{tabularx}{0.95\columnwidth}{>{\raggedright\arraybackslash}X >{\raggedright\arraybackslash}p{0.12\columnwidth} >{\raggedright\arraybackslash}p{0.18\columnwidth} >{\raggedright\arraybackslash}p{0.15\columnwidth}}
    \toprule
    \textbf{Method/Prompt} & \textbf{Edit Target} & \textbf{Post-Edit Output} & \textbf{Metrics} \\
    \midrule
  a) The genus Platypatrobus is part of the family? & Arctiinae & Arctiuc {\xmark} & Reliability{\xmark} \\
b) \textit{The genus Platypatrobus is a part of what family} & -- & Yemen \textcolor{red}{\xmark} & Generalization{\xmark} \\
\rowcolor{green}
c) The genus Platypatrobus is part of the family? & -- & Arctiinae {\checkmark} & \\
\midrule
c) \textit{When was the IAAF Combined Events Challenge launched?} & 2006 & Armand \textcolor{red}{\xmark} & Reliability{\xmark} \\
d) \textit{When does season 5 of ruby come out?} & October 14, 2017 & 2006 \textcolor{red}{\xmark} & Locality{\xmark} \\
\rowcolor{green}
e) \textit{when does season 5 of ruby come out?} & -- & 2017 {\checkmark} & \\

\bottomrule

\end{tabularx}
% \caption{\color{red}{Should this table have caption?}}
\end{table}

In summary, the main contributions are as follows.
\begin{itemize}[itemsep=0pt,parsep=0pt,topsep=0pt,partopsep=0pt]
\item We identify three critical challenges in model editing: (1) instability under large-scale sequential edits, (2) limited generalization in few-shot scenarios, and (3) inefficiency in open-loop, distribution-agnostic pipelines.
\item We propose REPAIR, a novel framework to address these challenges by integrating a dual-memory system with parametric editing. It introduces closed-loop error feedback, distribution-aware optimization, and loss-aware subspaces merging to ensure robust and precise updates.
\item We validate the performance of REPAIR across diverse models (including LLaMA-3, Qwen-2.5, DeepSeek‑R1‑1.5B, and GPT-2-XL), demonstrating a 15\%--20\% improvement in overall editing performance over state-of-the-art methods and showing consistent, robust generalization.
\end{itemize}

\section{Methodology}

We propose a novel closed-loop lifelong model editing framework, denoted \textbf{REPAIR},
which addresses the limitations of open-loop editing in distributed side-memory methods. Our framework, as shown in Figure \ref{fig:overall_structure}, integrates (1) closed-loop error feedback with dynamic memory management; (2) distribution-aware batch reassembly with inner-batch knowledge distillation; (3) loss-aware weighted knowledge merging.

\begin{figure*}
\centering
\includegraphics[width=0.9\linewidth]{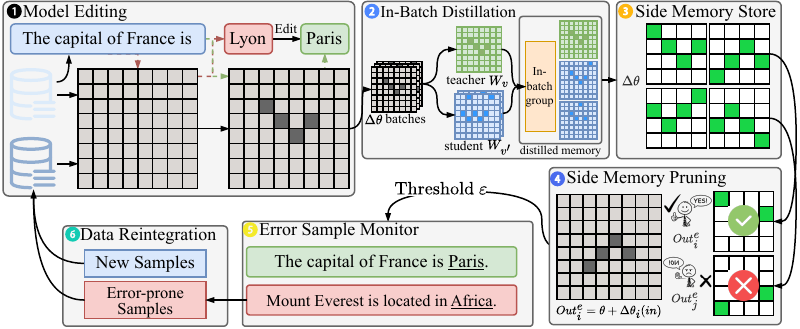}
\caption{\textbf{The overall structure of REPAIR.} An edit, such as changing the capital of France from "Lyon" to "Paris," is stored as a parameter update, $\Delta\theta$, in the Side Memory. An Error Sample Monitor evaluates the performance of each edit ($Out_i^e$). If the error rate, $Err_{thresh}$, for an edit on a new sample exceeds a threshold $\epsilon$, the Side Memory Pruning module removes the erroneous update. The system then reintegrates new and error-prone samples for continuous learning.}
\label{fig:overall_structure}
\end{figure*}

\subsection{Problem Setup}

% \begin{definition}[Lifelong Model Editing]
\begin{definition}[Lifelong Model Editing]
\label{def:lifelong_editing}
Given a pre-trained model \( f_{\theta_0}(y|x) \), a sequential edit stream \( \{\mathcal{E}_t\}_{t=1}^T \) where \( \mathcal{E}_t = \{(x_i^{(t)}, y_i^{(t)})\}_{i=1}^N \), and auxiliary distributions \( \mathcal{G}(x) \) (paraphrased inputs) and \( \mathcal{U} \) (unrelated contexts), the objective is to obtain updated parameters \( \theta_T \) that optimize the multi-objective trade-off:
{\small
\begin{equation}
\begin{aligned}
\theta_t = \arg\min_{\theta}\; &
\alpha\,\frac{1}{N}\sum_{i=1}^N \ell\bigl(f_\theta(\cdot\mid x_i^{(t)}), y_i^{(t)}\bigr) \\
&+ \beta\,\frac{1}{N}\sum_{i=1}^N \mathbb{E}_{x' \sim \mathcal{G}(x_i^{(t)})}
\Big[\ell\bigl(f_\theta(\cdot\mid x'), y_i^{(t)}\bigr)\Big] \\
&+ \gamma\,\mathbb{E}_{x \sim \mathcal{U}}
\Big[\mathrm{KL}\bigl(f_{\theta_{t-1}}(\cdot\mid x) \parallel f_\theta(\cdot\mid x)\bigr)\Big] \\
&+ R(\theta, \theta_{t-1}).
\end{aligned}
\end{equation}
}

where \( (\alpha, \beta, \gamma) \) are hyperparameters controlling the reliability-generalization-locality-stability trade-off, and \( R \) denotes a regularization term enforcing parameter smoothness across sequential edits.
\end{definition}

\subsection{Dual Memory Mechanism and Routing}

As shown in Figure \ref{fig:overall_structure}, block 1: For dual memory-based editing methods, the dual memory mechanism is typically deployed in the deep layers of the network.
Specifically, for the value matrix $\mathbf{W}_v$ of the target FFN layer, we create a copy as the side memory pool $M_s$, i.e.:
\(M_s^{(0)} = W_v.\)
If the side memory pool is activated, the output is computed as:
\(o_s = \phi(f^TW_k) \cdot M_s\),
where $\phi$ denotes the non-linear activation function, and $o_s$ represents the FFN output based on the side memory (\cite{Wang2024WISE}).

During the inference phase, for memory pool $i$, the activation score is defined as
\begin{equation}
\Delta^{(i)}_{\mathrm{act}}(x) = \| \mathcal{A}(x)\cdot (W_{v,i}' - W_v)\|_2.
\label{eq:activation-diff}
\end{equation}
where $\mathcal{A}(\cdot)$ = $a$ is the activation of the side memory's corresponding FFN layer. Routing selects the pool with the activation score.
If $\max_i \Delta^{(i)}_{\mathrm{act}}(x) \le \varepsilon$, the main memory $W_v$ is used. Otherwise, the side memory pool $M_s$ is selected. To enforce discriminative routing, we use a margin-based loss. The objective of the routing mechanism is to establish a clear decision boundary:
\begin{equation}
\min_{\mathbf{x}_e \sim \mathcal{E}} \mathcal{R}(\mathbf{x}_e)\sim\min_{\mathbf{x'} \sim \mathcal{U}} \mathcal{R}(\mathbf{x'}) > \tau > \max_{\mathbf{x}_i \sim \mathcal{G}} \mathcal{R}(\mathbf{x}_i)
\label{eq:activation-score}
\end{equation}
where $\tau$ is a preset threshold, and $\mathcal{E}$ and $\mathcal{G}_i$ represent the edit and edit-irrelevant datasets, respectively. This selective activation mechanism ensures that edited knowledge is only retrieved in relevant contexts, thereby minimizing interference with the original model's performance.

\subsection{Distribution-Aware inner-batch Knowledge Distillation}
As shown in Figure \ref{fig:overall_structure} block 2: A sample batch $\mathcal{E} = \{x_1, x_2, \dots, x_n\}$,
and we denote the feature representations by
\(o_i = \mathrm{Norm}(f_\theta(x_i)),\quad i=1,\dots,n.\)
To improve the consistency and stability of updates during edits, we organize samples into homogeneous batches and perform intrabatch knowledge distillation.
Samples with high mutual similarity are grouped into a batch $B = \{x^{(0)}, x^{(1)}, \dots, x^{(b-1)}\}$. Within each batch, the first sample $x^{(0)}$ acts as a \textit{teacher}, while the remaining samples are \textit{students}.
We define the inner-batch distillation loss as:
% 总损失函数
\begin{equation}
\mathcal{L}_{\text{kd}} =\lambda \cdot \mathcal{L}_{\text{cosine}} + \theta \cdot \mathcal{L}_{\text{variance}}
\end{equation}
where \(\mathcal{L}_{\text{cosine}} = 1 - \frac{o_i \cdot o_0}{\| o_i \| \| o_0 \|}\) and \(\mathcal{L}_{\text{variance}} =\frac{1}{N} \sum_{i=1}^{N} \| o_i - o_{\text{mean}} \|^2\).
Minimizing $\mathcal{L}_{\mathrm{kd}}$ encourages all samples in the batch to share similar knowledge, which in turn reduces potential conflicts when updating the same network parameters $\theta$. The regularization term is used to maintain diversity among features, preventing excessive uniformity.

If certain samples cannot be well-aligned with the batch (i.e., their $\mathcal{L}_{\mathrm{kd}}$ remains high after optimization), this indicates that they do not belong to the same distribution cluster and are unlikely to be effectively edited together. Such samples are removed from the batch and reclustered with other samples to form new homogeneous groups. Formally, the final batch reassembly can be expressed as
\begin{equation}
\mathcal{B}^* = \mathrm{Recluster}\big( \{ x \in B \mid \mathcal{L}_{\mathrm{kd}}(x, B) < \epsilon \} \big),
\end{equation}
where $\epsilon$ is a threshold controlling inner-batch consistency. This procedure ensures that sequential parameter edits are performed on groups of samples with aligned knowledge, improving both stability and effectiveness of the model update. The convergence proof is provided in the Appendix \ref{thm:rgd-sphere} and Appendix \ref{Finite-Time Convergence}.

\subsection{Closed-Loop Error Feedback and Memory Pruning}

As shown in Figure \ref{fig:overall_structure} block 4: After each editing cycle, we evaluate the performance in a feedback pool $\mathcal{E}$ of error response samples by comparing to the correctness threshold $\tau_{\text{correct}}$. For each shard $i$, we define the error set $\mathcal{E}_i = \{x \in \mathcal{E} \mid i^*(x) = i\}$ and compute the error rate $r_i^{\text{pool}}$ for each side memory pool, defined as the proportion of failed edits within the corresponding sample set:
\(r_i^{\text{pool}} = \frac{|\{x \in \mathcal{E}_i \mid a(x) \leq \tau_{\text{correct}}\}|}{|\mathcal{E}_i|}\)

When the pruning conditions are met ($r_i > \tau_{\text{prune}}$ or $|\mathcal{E}| > \tau_E$), we execute the following procedure:
\textbf{(1) Memory pool screening \& pruning:} Identify the side memory pool with the highest error rate $j = \arg\max_i r_i^{\text{pool}}$. Remove the identified memory pool from the system.
\textbf{(2) Sample reintegration \& retraining:} Recombine the remaining error samples to form a new training set $\mathcal{E}_{\text{retrain}}$. Retrain the new side memory pools using $\mathcal{E}_{\text{retrain}}$.
This closed-loop feedback mechanism enables the system to dynamically identify and eliminate underperforming memory units while optimizing the overall editing performance through sample reorganization and iterative retraining. The time-convergence proof is provided in the Appendix \ref{Finite-Time Convergence}.

\subsection{Merging with Weighted TIES}
As shown in Figure \ref{fig:overall_structure} block 3: After multiple updates, shards $\{W_{v,i}'\}$ produce deltas $\tau_i = W_{v,i}'-W_v$.
We merge them with the weighted TIES \cite{yadav2023tiesmerging} operator based on:
\(W_v' \leftarrow W_v + \omega_i\operatorname{TIES}\big(\{\tau_i\}_{i=1}^k; W_v\big)\).

The total loss integrates all components:
\begin{equation}
\mathcal{L}_{\mathrm{total}} = \mathcal{L}_{\mathrm{edit}}
+ \lambda_a \mathcal{L}_a
+ \lambda_{\mathrm{KD}}\mathcal{L}_{\mathrm{KD}}.
\end{equation}
\(\mathcal{L}_{\mathrm{edit}}\) is the autoregressive cross-entropy.
\(\mathcal{L}_{\mathrm{edit}}(W_v') = -\log P_{W_v'}(y\mid x)\).
To enforce discriminative routing, we use a margin-based loss:
\begin{equation}
\begin{aligned}
\mathcal{L}_a &= \min\Bigl\{ \max(0, \Delta_{\text{act}}(x_i) - \gamma_1) \\
&\quad + \max(0, \gamma_2 - \Delta_{\text{act}}(x_e)) \\
&\quad + \max\bigl(0, \gamma-(\Delta_{\text{act}}(x_e) - \Delta_{\text{act}}(x_i))\bigr) \Bigr\}.
\end{aligned}
\end{equation}
For shard $i$, consider $k$ subspaces $\{\theta_1, \dots, \theta_k\}$, each trained on a subset of samples $\mathcal{E}_i$. Let the average training loss of subspace $\theta_i$ be
\(\mathcal{L}_i = \frac{1}{|\mathcal{E}_i|} \sum_{(x,y) \in \mathcal{E}_i} \ell(f(x;\theta_i), y)\),
where $\ell(\cdot)$ is the task loss. We define the merging weight of each subspace as \(w_i = \frac{\exp(-\alpha \mathcal{L}_i)}{\sum_{j=1}^M \exp(-\alpha \mathcal{L}_j)}\), with $\alpha>0$ controlling sensitivity to the loss. The global network parameters are then obtained via weighted averaging:
\(\theta = \sum_{i=1}^M w_i \, \theta_i\). This loss-aware merging favors subspaces that achieve lower training loss on their corresponding samples, promoting reliable knowledge integration.

\section{Experiments}
% In the experimental section, we design six evaluations to answer the following questions:
% \begin{itemize}
%     \item \textbf{Q1}: Do the three key innovations (closed-loop feedback, discriminative pruning, and distribution reintegration) improve edit accuracy, generalization, and locality?
%     \item \textbf{Q2}: Does the method generalize well to knowledge-intensive tasks such as question answering and hallucination mitigation?
%     \item \textbf{Q3}: Is the method effective across different parameter scales and diverse architectures, including recent open-source models?
%     \item \textbf{Q4}: Under distribution shift (e.g., on the Wikibig Edit dataset), does the method remain robust and outperform existing methods?
%     \item \textbf{Q5}: Can the method maintain long-term stability and reliability in large-scale sequential editing scenarios?
%     \item \textbf{Q6}: What are the contributions and sensitivities of each component and hyperparameter to overall performance?
% \end{itemize}
In the experimental section, we design six evaluations to answer the following questions: \textbf{Q1}: Do the three key innovations (closed-loop feedback, discriminative pruning, and distribution reintegration) improve edit accuracy, generalization, and locality? \textbf{Q2}: Does the method generalize well to knowledge-intensive tasks such as question answering and hallucination mitigation? \textbf{Q3}: Is the method effective across different parameter scales and diverse architectures, including recent open-source models? \textbf{Q4}: Under distribution shift (e.g., on the WikiBigEdit dataset), does the method remain robust and outperform existing methods? \textbf{Q5}: Can the method maintain long-term stability and reliability in large-scale sequential editing scenarios? \textbf{Q6}: What are the contributions and sensitivities of each component and hyperparameter to overall performance?

% \textbf{1).} we measure the impact of its three key innovations—error-driven closed-loop feedback, discriminative pruning in the side memory, and identification and reintegration of original data distributions—across models with varying scales (dozens to thousands of edits), using metrics such as edit accuracy, generalization, and locality. \textbf{2).}we test generalization in knowledge-intensive tasks including question answering and hallucination mitigation. \textbf{3).} we evaluate applicability across diverse architectures and parameter scales, including recent open-source GPT models. \textbf{4).} we compare against WISE on the Wikibig Edit dataset to validate robustness under distribution shift. \textbf{5).} large-scale sequential editing experiments demonstrate stability and reliability under mass edits. \textbf{6).} ablation studies analyze the contribution of each component and hyperparameter sensitivity. As in WISE, we augment edited instances $x_e$ with random token sequences from the main memory network to enhance contextual generalization, a strategy previously validated in WISE.

\subsection{Experimental Setup}
\textbf{Datasets and Models}. Autoregressive LLMs are ideal for evaluating model editing due to their unidirectional causal structure, which allows predictable and traceable edits. This ensures clear interpretability of edit generalization and locality. We evaluate widely used models (LLaMA-3-8B, GPT-2-XL) and recent models (Qwen2.5-7B, DeepSeek‑R1‑1.5B), using datasets such as ZsRE for closed-book QA, WikiBigEdit for editing performance, and a hallucination dataset to assess generalization. For more details, refer to Appendix \ref{Dataset statistics}.

\noindent\textbf{Baselines.} We consider three families of methods: \textbf{Direct Parameter Editors} that directly modify model weights (e.g., \textbf{ROME} \cite{Gupta2024RebuildROME}, \textbf{MEMIT} \cite{Meng2023MEMIT}, \textbf{MEMIT-mass} \cite{Meng2023MEMIT}); \textbf{Hypernetwork-Based Editors} that use an auxiliary network to generate parameter updates at inference (e.g., \textbf{MEND} \cite{Mitchell2022MEND}); and \textbf{External Memory-Based Editors} that store edits in external memory and retrieve them via a routing mechanism (e.g., \textbf{SERAC} \cite{Mitchell2022SERAC}, \textbf{GRACE} \cite{Hartvigsen2023GRACE}, \textbf{WISE} \cite{Wang2024KESurvey}).

\noindent\textbf{Implementation Details.} Experiments were conducted simultaneously using two GPUs: an A100 PCIe 80GB and an A100 SXM4 40GB. The code was implemented based on PyTorch 2.1, with modifications built upon the original EasyEditor framework. We provide the code to facilitate reproducibility at \url{https://anonymous.4open.science/r/REPAIR}. The specific hyperparameter settings are detailed in Appendix \ref{app:experiment_details}.

\noindent\textbf{Evaluation Metrics.}
Each edited corpus instance comprises three components: the descriptor $k_e$ used to perform the edit, an irrelevant prompt-answer pair $k_e'$ to verify locality and a rephrase prompt $k_{loc}$ to evaluate generalization performance across different expressions. To comprehensively evaluate the optimization capability of the proposed method in addressing the continual learning trilemma, we employ four metrics—edit accuracy:
\(\textbf{Rel} = \frac{1}{N} \sum_{n=1}^{N} \mathbf{l}(f_{\omega_N}(\mathbf{x}_e^n) = \mathbf{y}_e^n)\),
rephrase accuracy :
\(\textbf{Gen} = \frac{1}{N} \sum_{n=1}^{N} \mathbf{l}(f_{\omega_N}(\mathbf{x'}_e^n) = \mathbf{y}_e^n)\),
locality :
\(\textbf{Loc} = \frac{1}{N} \sum_{n=1}^N \mathbf{l}(f_{\omega_N}(\mathbf{x}_{\text{loc}}^n) = f_{\omega_0}(\mathbf{x}_{\text{loc}}^n)\).
We use the geometric mean of Rel., Gen., and Loc. to evaluate the overall editing performance, which balances metric sensitivity and interpretability, exhibits sensitivity to weak performance areas, and is suitable for scenarios where all three metrics are equally important.
\(\textbf{OP} = \sqrt[3]{\text{Rel.} \times \text{Gen.} \times \text{Loc.}}\)
to assess the holistic editing effectiveness. Here, $\mathbf{l}(\cdot)$ is the indicator function used to count the number of successful predictions.

For the hallucination dataset specifically, we utilize perplexity (PPL) as the metric to assess editing performance. PPL can be interpreted as the "average branching factor in predicting the next token," where a lower value indicates more accurate model predictions and suggests a reduced likelihood of the edited model generating hallucinations.
\(\textbf{PPL} = \exp\left(-\frac{1}{N} \sum_{i=1}^{N} \log P(y_i | \text{context}_i)\right)\)

\subsection{Main Results}

\begin{table*}[h]
\caption{Comparative results for QA on multi-scale editing (ZsRE and WikiBigEdit) $N$: Num Edits.}
\label{Table 3}
\centering
\small
\setlength{\tabcolsep}{2pt}
\begin{tabular}{lccc|c|ccc|c|ccc|c|ccc|c|}
\toprule
\multirow{2}{*}{\textbf{Method}} & \multicolumn{4}{c}{$N = 1$} & \multicolumn{4}{c}{$N = 30$} & \multicolumn{4}{c}{$N = 120$} & \multicolumn{4}{c}{$N = 1000$} \\
\cmidrule{2-17}

& Rel. & Gen. & Loc. & OP. & Rel. & Gen. & Loc. & OP. & Rel. & Gen. & Loc. & OP. & Rel. & Gen. & Loc. & OP. \\
\midrule
& \multicolumn{15}{c}{LLaMA-3-8B (ZsRE)} \\
\cmidrule{1-17}
FT-L & 0.57 & 0.52 & 0.96 & 0.66 & 0.35 & 0.35 & 0.52 & 0.39 & 0.29 & 0.26 & 0.21 & 0.25 & 0.19 & 0.15 & 0.02 & 0.08 \\
FT-EWC & 0.96 & \textbf{0.93} & 0.02 & 0.26 & 0.78 & 0.76 & 0.02 & 0.23 & 0.76 &\textbf{0.76}& 0.08 & 0.36 & 0.69 &\textbf{0.67} & 0.08 & 0.33 \\
MEND & 0.95 & 0.93 & 0.96 & 0.95 & 0.24 & 0.25 & 0.18 & 0.22 & 0.08 & 0.07 & 0.00 & 0.00 & 0.00 & 0.00 & 0.00 & 0.00 \\
ROME & 0.85 & 0.80 & 0.99 & 0.88 & 0.61 & 0.60 & 0.68 & 0.63 & 0.22 & 0.22 & 0.04 & 0.12 & 0.01 & 0.01 & 0.01 & 0.01 \\
MEMIT-M & 0.84 & 0.81 & 0.99 & 0.88 & 0.73 & 0.72 & 0.95 & 0.79 & 0.70 & 0.65 & 0.82 & 0.72 & 0.63 & 0.63 & 0.62 & 0.63 \\
DEFER & 0.68 & 0.58 & 0.56 & 0.61 & 0.65 & 0.47 & 0.36 & 0.49 & 0.20 & 0.12 & 0.27 & 0.20 & 0.03 & 0.03 & 0.74 & 0.27 \\
GRACE & \textbf{0.97} & 0.36 & \textbf{1.00} & 0.71 & \textbf{0.96} & 0.17 & \textbf{1.00} & 0.55 & \textbf{0.94} & 0.14 & \textbf{1.00} & 0.51 & \textbf{0.93} & 0.08 & \textbf{1.00} & 0.42 \\
WISE & 0.94 & 0.92 & \textbf{1.00} & \textbf{0.95} & 0.62 & 0.60 & 0.86 & 0.68 & 0.57 & 0.58 & 0.87&0.66& 0.45 & 0.44 & 0.51 & 0.47\\
\cmidrule{1-17}
\rowcolor{lightpurple}
\textbf{REPAIR} & 0.94 & 0.92 & \textbf{1.00} & \textbf{0.95} & 0.93 & \textbf{0.90} & 0.87 & \textbf{0.89$\uparrow$} & 0.76 & 0.74 & \textbf{1.00} & \textbf{0.83$\uparrow$} & 0.68 & 0.65 & 0.89 & \textbf{0.73$\uparrow$} \\
\midrule
& \multicolumn{15}{c}{Qwen2.5-7B (ZsRE)} \\
\cmidrule{1-17}
FT-L & 0.68 & 0.63 & 0.93 & 0.74 & 0.28 & 0.23 & 0.44 & 0.30 & 0.13 & 0.11 & 0.10 & 0.11 & 0.08 & 0.06 & 0.02 & 0.05 \\
FT-EWC & 0.97 & 0.92 & 0.05 & 0.35 & 0.82 & 0.80 & 0.02 & 0.24 & 0.71 & 0.69 & 0.05 & 0.29 & 0.58 & 0.56 & 0.03 & 0.21 \\
MEND & 0.96 &\textbf{0.95} & 0.96 & 0.96 & 0.31 & 0.31 & 0.27 & 0.29 & 0.15 & 0.14 & 0.03 & 0.09 & 0.02 & 0.02 & 0.00 & 0.00 \\
ROME & 0.90 & 0.89 & 0.99 & 0.93& 0.77 & 0.73 & 0.52 & 0.66 & 0.31 & 0.28 & 0.03 & 0.14 & 0.01 & 0.02 & 0.00 & 0.00\\
MEMIT-M & 0.84 & 0.81 & 0.99 & 0.88 & 0.73 & 0.72 & 0.95 & 0.79 & 0.70 & 0.65 & 0.82 & 0.72 & 0.52 & 0.51 & 0.57 & 0.53 \\
DEFER & 0.74 & 0.67 & 0.88 & 0.76 & 0.58 & 0.51 & 0.44 & 0.51 & 0.22 & 0.21 & 0.43 & 0.27 & 0.14 & 0.08 & 0.25 & 0.14 \\
GRACE & 0.97 & 0.41 & 0.98 & 0.73 & \textbf{0.97} & 0.2 & \textbf{1.00} & 0.58 & \textbf{0.95} & 0.08 & \textbf{0.98} & 0.42 & \textbf{0.94} & 0.02 & \textbf{1.00} & 0.27 \\
WISE & 0.97 & \textbf{0.95} & 0.98 & 0.97 & 0.79 & 0.73 & 0.91 & 0.80 & 0.59 & 0.57 & 0.92&0.68& 0.44 & 0.41& 0.72 & 0.51\\
\cmidrule{1-17}
\rowcolor{lightpurple}
\textbf{REPAIR} &\textbf{0.98}  &\textbf{0.95} & \textbf{1.00} & \textbf{0.98 $\uparrow$} & 0.93 & \textbf{0.90} & 0.93 & \textbf{0.92$\uparrow$} & 0.81& \textbf{0.80} & 0.92 & \textbf{0.84$\uparrow$} & 0.72 & \textbf{0.70} & \textbf{0.67} & \textbf{0.69$\uparrow$} \\
\midrule
& \multicolumn{15}{c}{DeepSeek‑R1‑1.5B (WikiBigEdit) } \\
\cmidrule{1-17}
FT-L        & 0.71 & 0.68 & 0.93 & 0.77 & 0.26 & 0.20 & 0.76 & 0.34 & 0.13 & 0.11 & 0.37 & 0.17 & 0.02 & 0.02 & 0.08 & 0.03\\
FT-EWC      & 0.93 & 0.91 & 0.33 & 0.65 & 0.70 & 0.70 & 0.18 & 0.45 & 0.42 & 0.41 & 0.07 & 0.23 & 0.18 & 0.15 & 0.02 & 0.08\\
MEND        & 0.91 & 0.87 & 0.95 & 0.91 & 0.43 & 0.38 & 0.10 & 0.25 & 0.24 & 0.23 & 0.08 & 0.16& 0.03 & 0.03 & 0.02 & 0.05 \\
ROME        & 0.86 & 0.83 & 0.97 & 0.88 & 0.72 & 0.71 & 0.67 & 0.70 & 0.18 & 0.18 & 0.02 & 0.09& 0.01 & 0.0 & 0.01 & 0.00\\
MEMIT-M  & 0.86 & 0.87 & 0.97 & 0.90 & 0.78 & 0.77 & 0.82 & 0.79 & 0.54 & 0.51 & 0.77 & 0.60 & 0.38 & 0.38 & 0.62 & 0.45\\
DEFER       & 0.68 & 0.58 & 0.47 & 0.35 & 0.63 & 0.61 & 0.51 & 0.58 & 0.17 & 0.15 & 0.33 & 0.20 & 0.07 & 0.07 & 0.12 & 0.08\\
GRACE       & 0.96 & 0.47 & \textbf{0.99} & 0.76 & \textbf{0.93} & 0.24 & \textbf{0.91}& 0.59 & \textbf{0.76} & 0.13 & 0.89 & 0.44 & 0.63 & 0.07 & \textbf{0.81} & 0.33\\
WISE        & 0.89 & 0.91 & 0.98 & 0.93 & 0.76 & 0.74 & 0.89 & 0.79 & 0.64 & 0.65 & 0.83 & 0.70 & 0.47 & 0.38 & 0.61 & 0.48\\
\midrule
\rowcolor{lightpurple}
\textbf{REPAIR}     & \textbf{0.98} & \textbf{0.93} & 0.98 & \textbf{0.96$\uparrow$} & 0.84 & \textbf{0.83} & \textbf{0.91} & \textbf{0.86$\uparrow$} & 0.71 & \textbf{0.69} & \textbf{0.90} & \textbf{0.76$\uparrow$} & \textbf{0.58} & \textbf{0.54} & \textbf{0.81} & \textbf{0.63$\uparrow$}\\
\bottomrule
\end{tabular}
\end{table*}

\begin{figure}[t]
\centering
\includegraphics[width=\columnwidth]{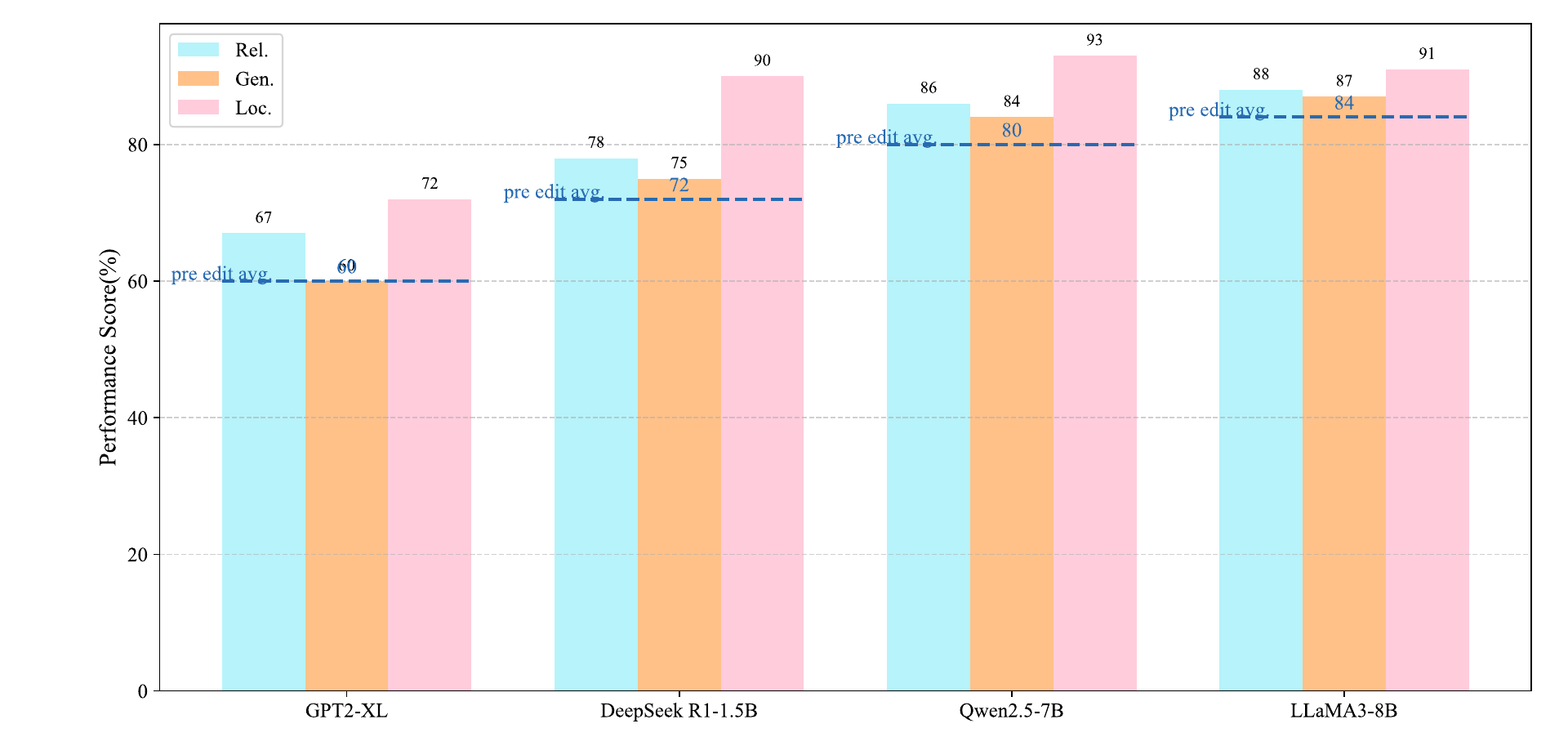}
\caption{\footnotesize Average Editing Performance of WikiBigEdit Across Different Models}
\label{fig:wiki_performance}
\end{figure}

Table \ref{Table 3} effectively addressed \textbf{Q1}, \textbf{Q4} and \textbf{Q5}.  It has been rigorously evaluated across diverse models and scales (N = 1, 30, 120, 1000) of QA editing tasks, demonstrating state-of-the-art performance. Fine-tuning-based methods achieve good accuracy and generalization at small scales but suffer from catastrophic forgetting and knowledge conflicts in large-scale edits, leading to performance degradation. GRACE excels in accuracy but has limited generalization, while WISE maintains strong locality but sacrifices critical knowledge, reducing editing accuracy. ROME-style methods are stable but overfit and struggle with generalization.

To address \textbf{Q2}, Table \ref{Hallutask} shows REPAIR's effectiveness in reducing hallucinations on the SelfCheckGPT dataset for LLaMA-3-8B across different editing scales. REPAIR balances reduced hallucinations with preserved locality, making it highly effective for large-scale model editing.

%%Hallucination实验表格
\begin{table*}[h]\scriptsize
\centering
\small
\caption{Main editing results for Hallucination task (SelfCheckGPT).}
\begin{tabular}{lcccccccccccc}
\toprule
\multirow{2}{*}{\textbf{Method}} & \multicolumn{2}{c}{$N = 1$} & \multicolumn{2}{c}{$N = 30$} & \multicolumn{2}{c}{$N = 120$} & \multicolumn{2}{c}{$N = 500$} \\
\cmidrule{2-9}
&Rel. (\textit{PPL} $\downarrow$) & Loc. ($\uparrow$) &
Rel. ($\downarrow$) & Loc. ($\uparrow$) &
Rel. ($\downarrow$) & Loc. ($\uparrow$) &
Rel. ($\downarrow$) & Loc. ($\uparrow$) \\
\midrule
\multicolumn{9}{c}{\textbf{LLaMA-3-8B}}\\
\midrule
FT-L        & 4.27   & 0.96 & 3.15    & 0.71 & 34.52   & 0.43 & 51.31   & 0.26 \\
FT-EWC      & 2.18   & 0.24 & 3.51    & 0.09 & 2.90    & 0.21 & 3.48    & 0.24 \\
MEND        & 5.34   & 0.87 & 1.24    & 0.86 & 9.17   & 0.89 & 564.9 & 0.00 \\
ROME        & 1.88   & 0.99 & 2.47   & 0.94 & 84.56   & 0.03 & 73.4  & 0.02 \\
MEMIT-M & 1.62   & \textbf{1.00} & 1.78 & 0.99 & 8.03  & 0.99 & 7.43  & 0.94 \\
DEFER       & \textbf{1.29} & 0.23 & 4.12 & 0.28 & 8.91  & 0.19 & 15.16  & 0.12 \\
GRACE       & 2.21   & \textbf{1.00} & 8.67 & \textbf{1.00} & 7.24 & \textbf{1.00} & 6.18 & \textbf{1.00} \\
WISE        &1.91 & \textbf{1.00} &1.59& \textbf{1.00} &1.14 & 0.99 &2.08 & 0.99 \\
\midrule
\rowcolor{lightpurple}
\textbf{REPAIR}& 1.43& \textbf{1.00} &1.37& \textbf{1.00}&\textbf{1.12} & \textbf{1.00} &\textbf{1.91} & \textbf{1.00} \\
\bottomrule
\label{Hallutask}
\end{tabular}
\end{table*}

\noindent\textbf{Reasoning ability retention (MMLU).} To verify that sequential editing does not compromise general reasoning capability, we evaluate the model before and after $N=100$ edits on MMLU (57 subjects). MMLU accuracy changes only slightly (0.6591 $\rightarrow$ 0.6332), indicating limited degradation of general reasoning ability.

To address \textbf{Q3} and \textbf{Q4}, Table~\ref{Table 3} and Figure~\ref{fig:wiki_performance} show that REPAIR’s closed‑loop error feedback, with distribution‑aware clustering and redistribution, yields superior performance across edit scales and stability for large‑scale edits. Smaller models concentrate knowledge in narrower parameter subsets, enabling reliable local corrections but weakening long‑term stability and generalization (i.e., maintaining accuracy while preserving unrelated knowledge). Accordingly, DeepSeek‑R1‑1.5B attains higher correction rates at small edit\_Num, yet degrades quickly as N grows. For locality, LLaMA‑3‑8B and Qwen2.5‑7B are marginally stronger due to parameter redundancy; DeepSeek‑R1‑1.5B remains competitive only at low N, then collapses under extreme multi‑point editing. In contrast, larger models distribute knowledge more broadly, and though harder to modify—successful edits generalize better across contexts. At a medium scale (N=120), MEMIT‑M and WISE show Rel., likely because REPAIR’s pruning / reassembly introduces transient instability before sufficient error signals accumulate; however, at N=1000 their performance drops sharply, while REPAIR’s dynamic adjustment preserves robustness and achieves the best metric. The error distribution is reported in Table~\ref{Tab:with-err}.

%% 蒸馏前后激活指标散点图%%

\begin{figure*}[!b]
\centering
\subfloat[]{
\includegraphics[width=0.48\textwidth]{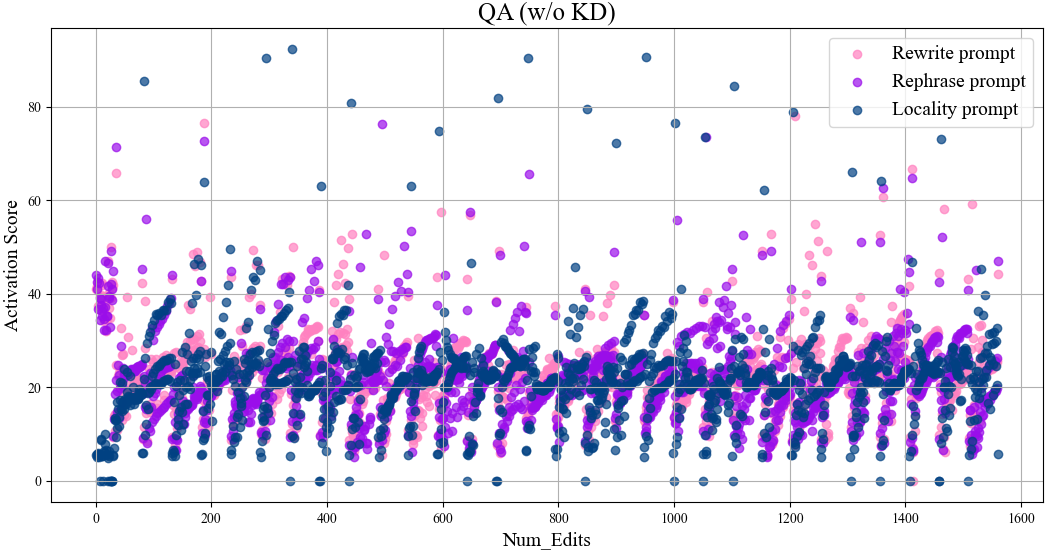}
\label{fig:sub1}
}
\hfill
\subfloat[]{
\includegraphics[width=0.48\textwidth]{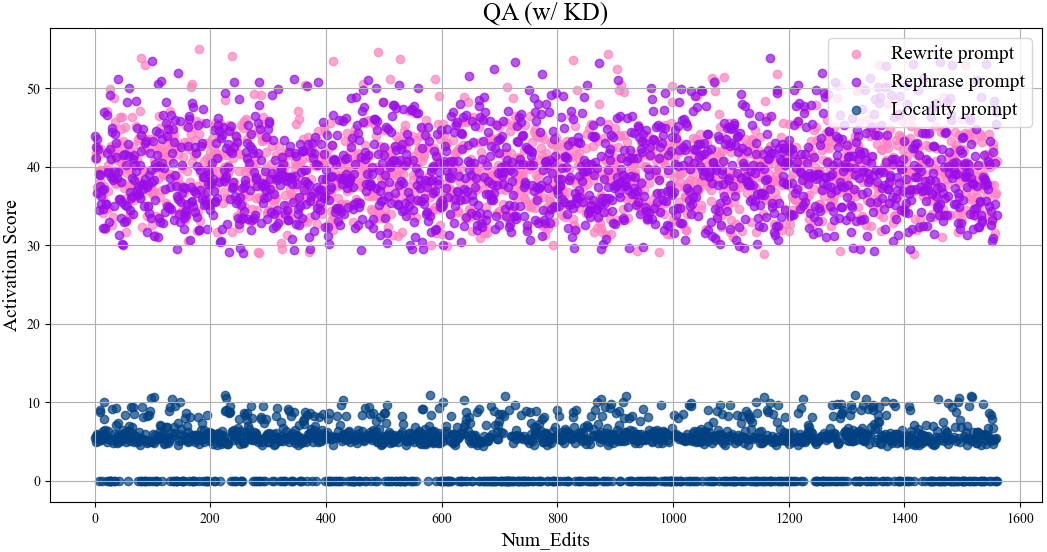}
\label{fig:sub2}
}
\hfill
\subfloat[]{
\includegraphics[width=0.48\textwidth]{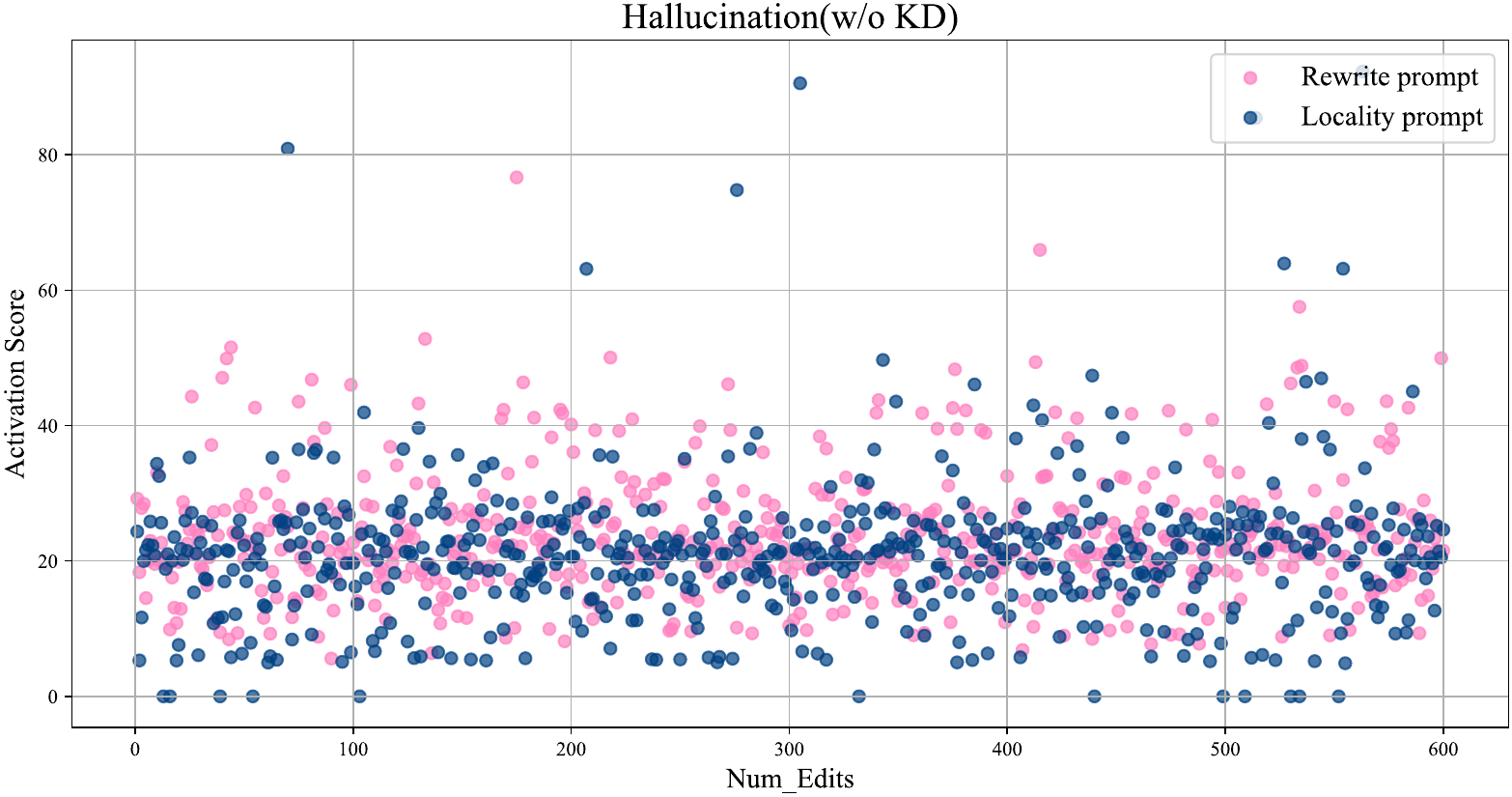}
\label{fig:sub3}
}
\hfill
\subfloat[]{
\includegraphics[width=0.48\textwidth]{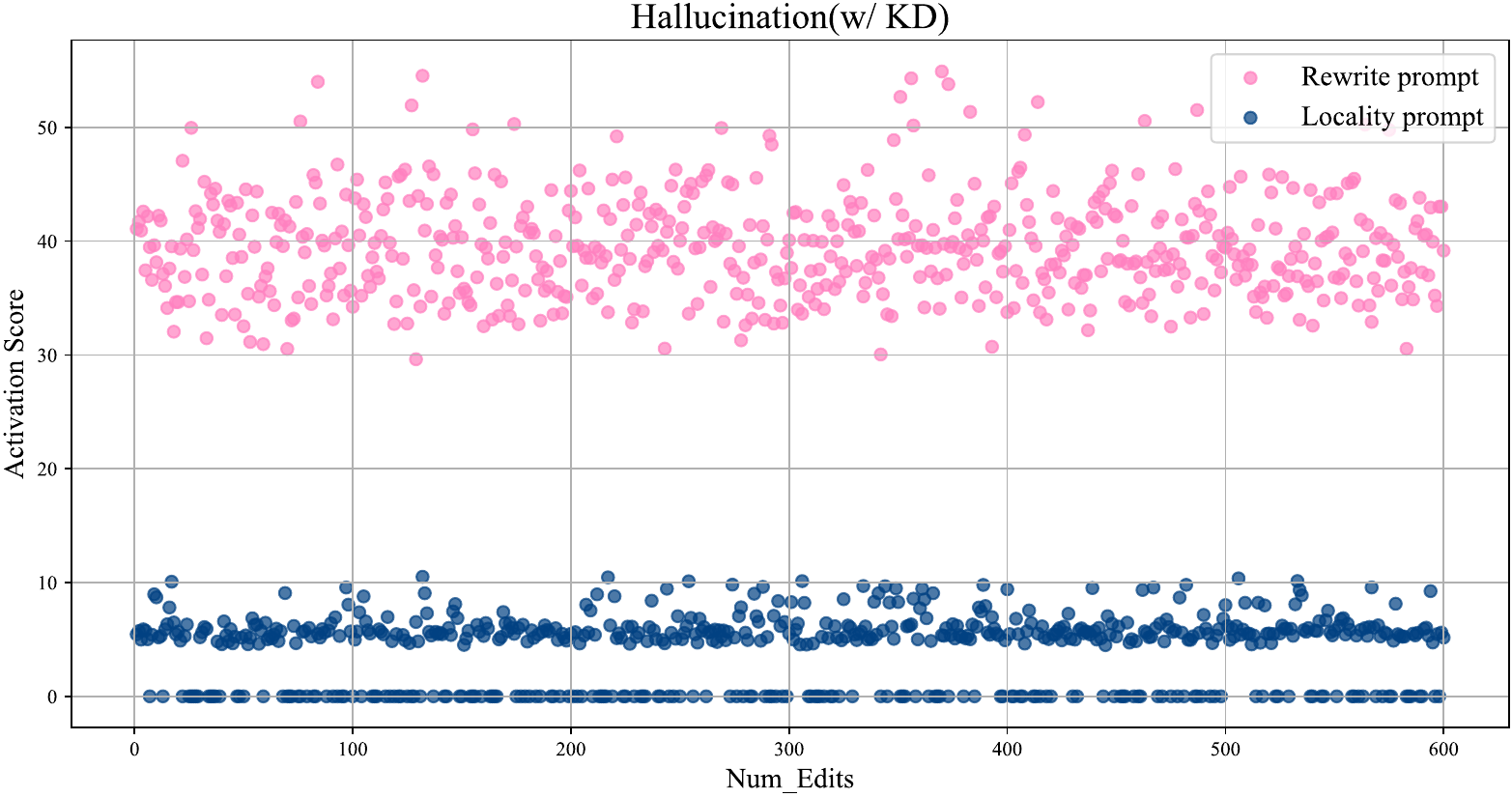}
\label{fig:sub4}
}
\caption{\textbf{Activation Score Visualization}. Results on LLaMA‑3 for the WikiBigEdit dataset (N=1550) for the QA task and the SelfCheckGPT dataset for hallucination (N=600).}
\label{fig:both}
\end{figure*}

Figure \ref{eq:activation-score} further addresses \textbf{Q1} regarding the effectiveness of distillation. For external memory-based editors, the ability to select the correct network for inference directly determines editing performance. The activation score, which serves as a critical routing criterion in memory networks, must exhibit statistically significant differences between new knowledge and irrelevant knowledge to ensure both reliability and locality of edits. As shown in Figure \ref{fig:both} (a) and (c), prior methods relying solely on triple-boundary loss fail to adequately separate the activation scores of $\textit{Data}{edit}$, $\textit{Data}{rephrase}$, and $\textit{Data}_{loc}$, particularly in large-scale continual editing scenarios, leading to a breakdown of the routing mechanism. This deficiency fundamentally limits their editing performance. In contrast, by introducing inner-batch knowledge distillation, sample filtering, and samples reintegration, KD, as shown in Figure \ref{fig:both} (b) and (d), achieves a clear separation among the three types of samples, thereby ensuring the proper functioning of the routing mechanism.

\subsection{Ablation Studies}
\begin{figure*}[!t]
\centering
% 四个雷达图在一行
\begin{subfigure}[b]{0.24\textwidth}
\centering
\includegraphics[width=\textwidth]{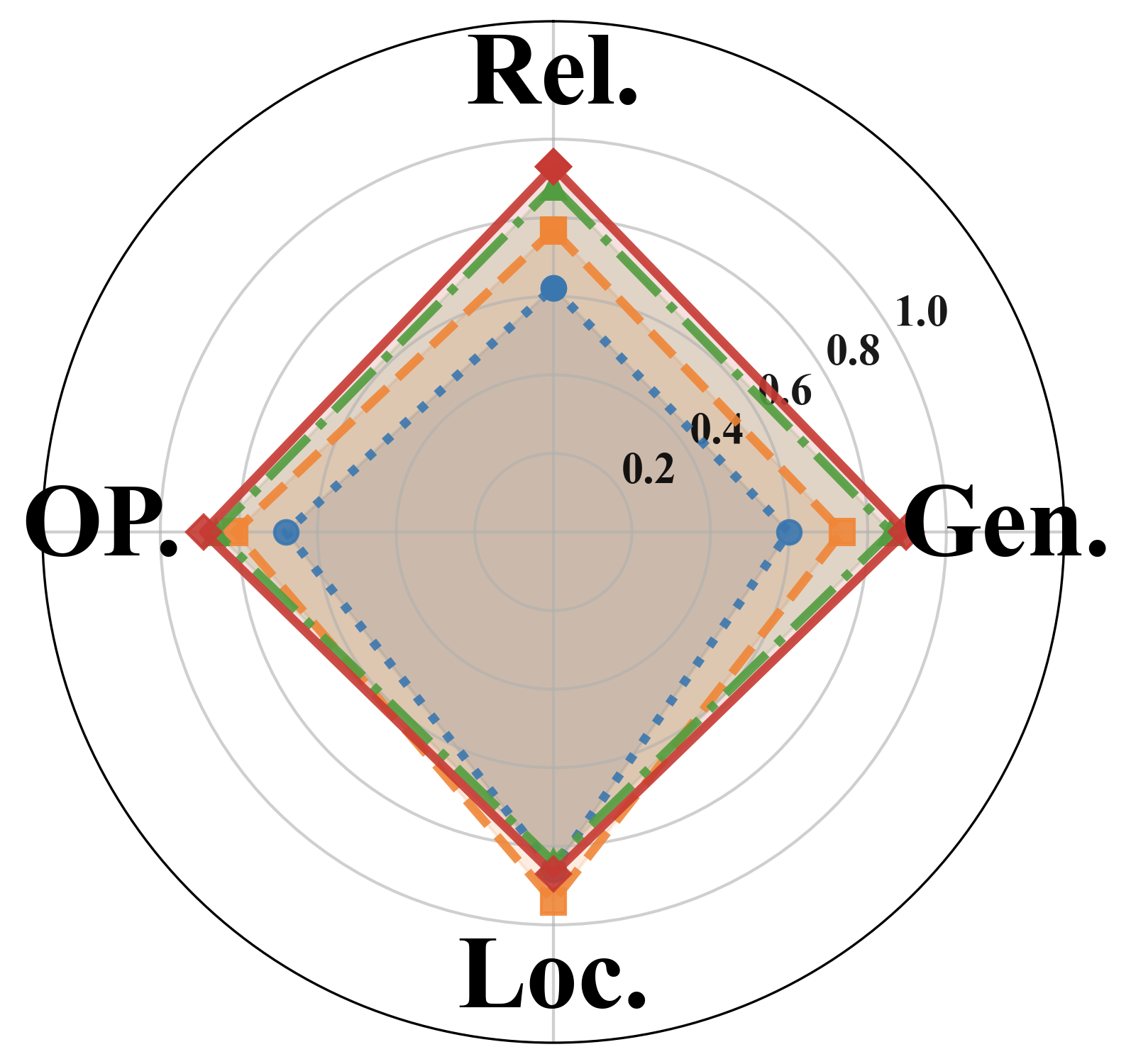}
\caption{N=30}
\end{subfigure}
\hfill
\begin{subfigure}[b]{0.24\textwidth}
\centering
\includegraphics[width=\textwidth]{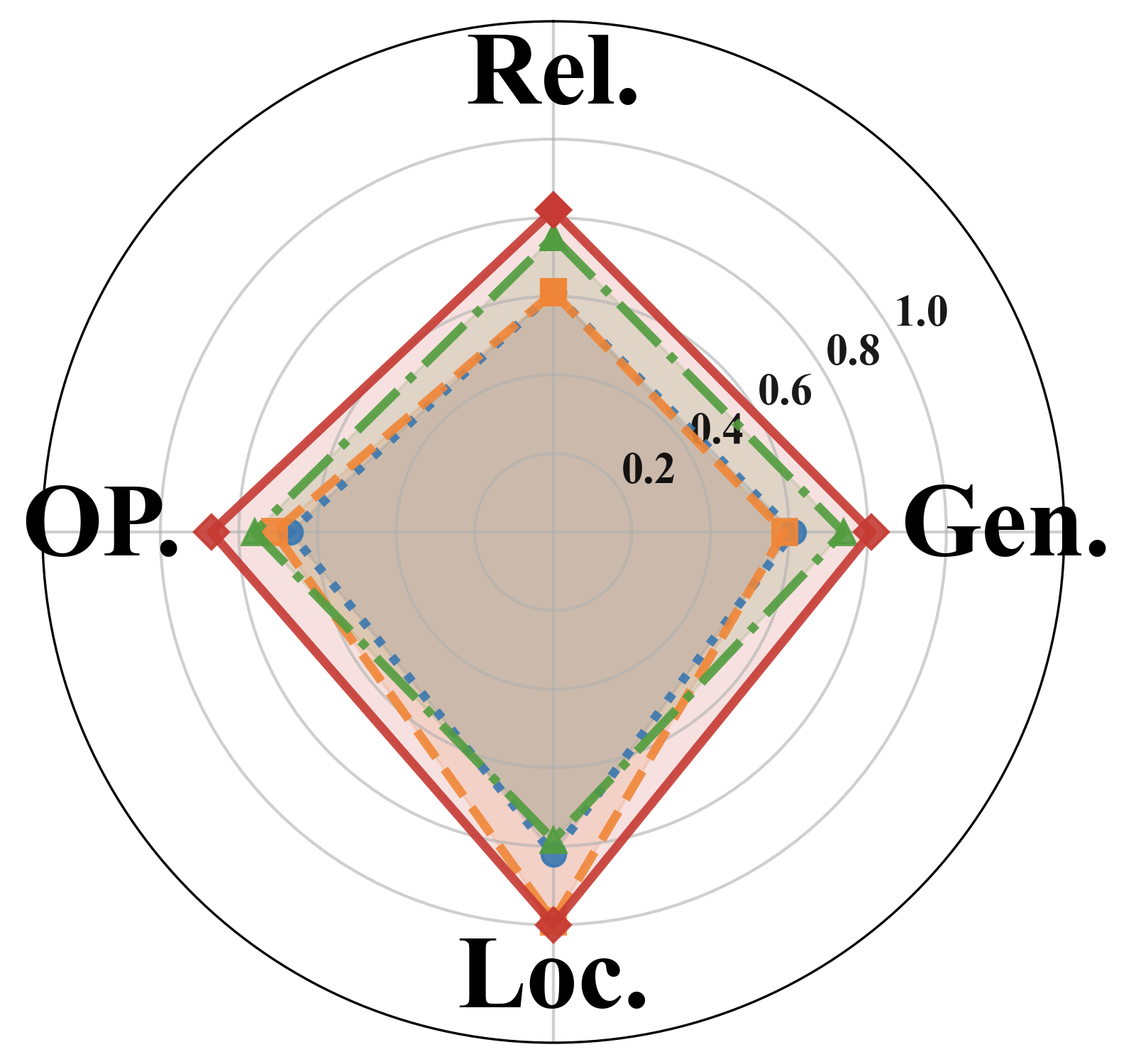}
\caption{N=60}
\end{subfigure}
\hfill
\begin{subfigure}[b]{0.24\textwidth}
\centering
\includegraphics[width=\textwidth]{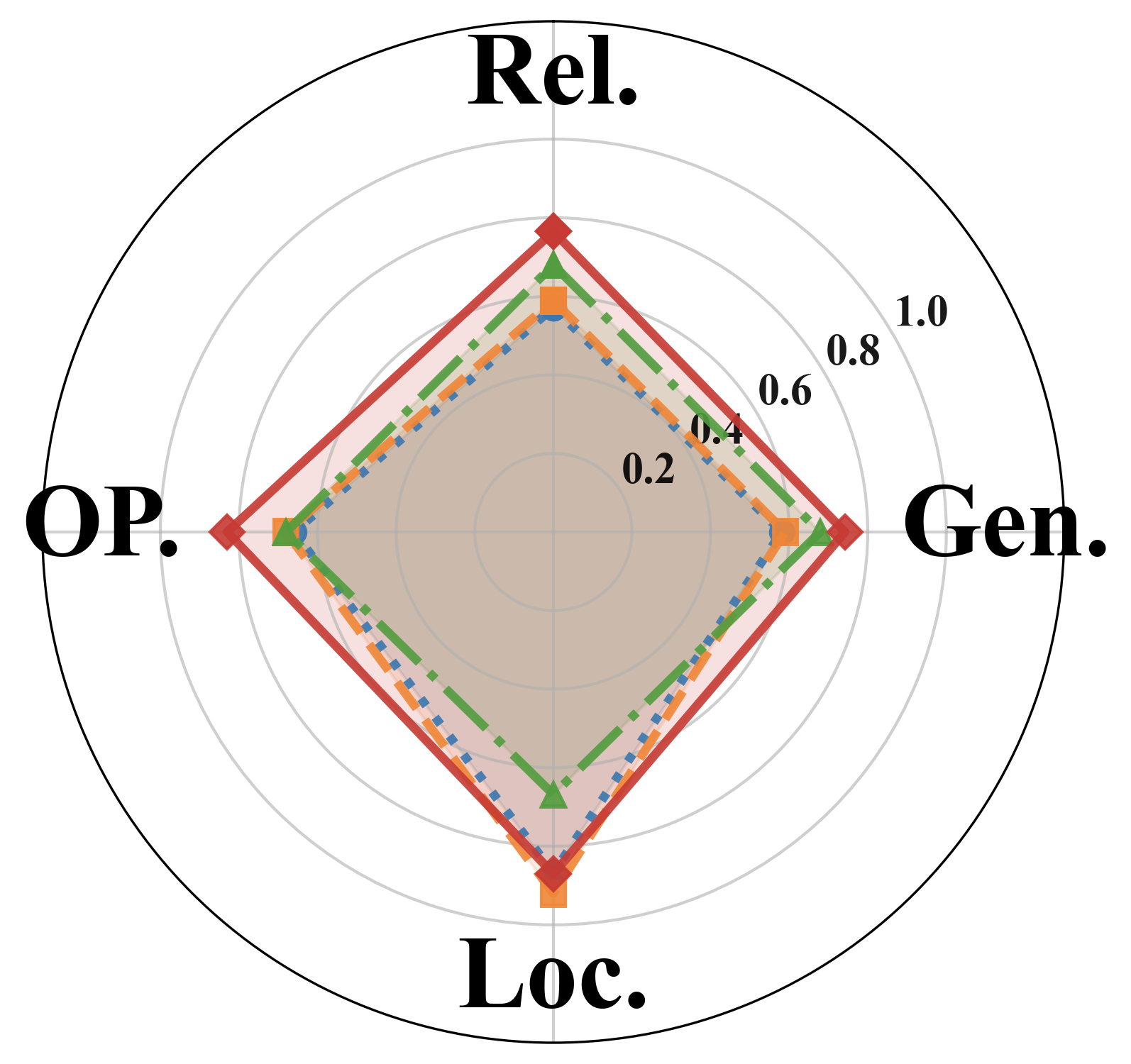}
\caption{N=120}
\end{subfigure}
\hfill
\begin{subfigure}[b]{0.24\textwidth}
\centering
\includegraphics[width=\textwidth]{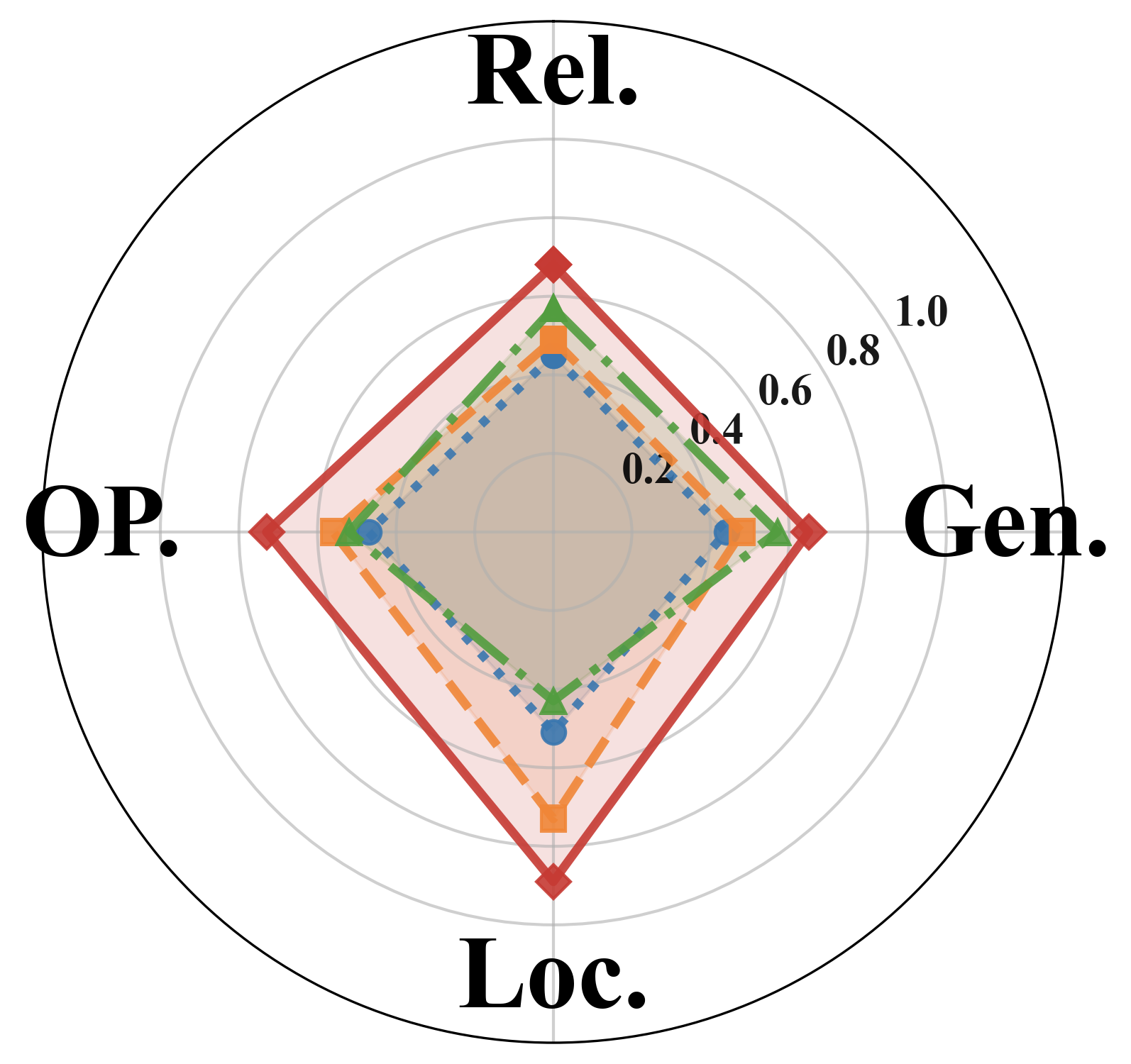}
\caption{N=1000}
\end{subfigure}

% 图例在下面
\begin{subfigure}[b]{\textwidth}
\centering
\includegraphics[width=0.8\textwidth]{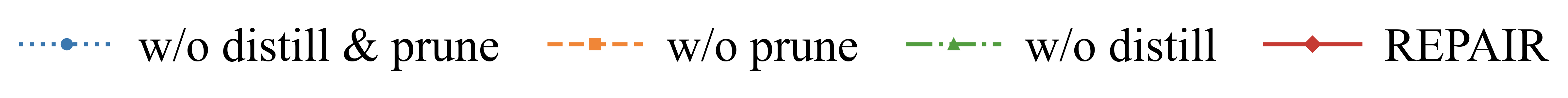}
\end{subfigure}

\caption{\textbf{Performance comparison of different components.} Each radar chart shows performance on four metrics: Rel., gen., loc., and OP. on Qwen2.5 with ZsRE.}
\label{fig:radar_comparison}
\end{figure*}

The evaluation of REPAIR's overhead and throughput is shown in \ref{fig:timeconsuming}. To answer \textbf{Q6}, we conducted evaluations across four dimensions to assess the effectiveness of each component of REPAIR and analyze hyperparameter sensitivity under different editing scales. Notably, REPAIR demonstrates robustness in large-scale editing scenarios that prior methods fail to achieve. As the number of edits increases, REPAIR exhibits increasingly pronounced advantages in overall performance: effective routing ensures strong locality, while the error-feedback mechanism maintains continual reliability. As shown in Figure \ref{fig:radar_comparison} (a)–(d), the relative contributions of REPAIR's components vary across sample regimes but complement each other seamlessly. In small-scale edits, pruning with error feedback improves reliability, while in large-scale scenarios, distribution-aware recognition and knowledge distillation become more critical.

\noindent\textbf{Homogeneous vs. heterogeneous batching.} To validate that distribution-aware batching is necessary (rather than merely increasing optimization steps), we disable similarity-based clustering and force the editor to train on heterogeneous batches. As shown in Table~\ref{tab:hetero_batch}, heterogeneous batching improves rewrite accuracy but significantly harms locality, confirming that our clustering mitigates interference among unrelated edits. Besides, we validate that the feature representations used for similarity-based batching capture meaningful within-task similarity (Appendix~\ref{app:feat_sim}).
\begin{table}[t]
\centering
\caption{\textbf{Effect of heterogeneous batches} (Qwen2.5, QA).}
\label{tab:hetero_batch}
\setlength{\tabcolsep}{8pt}
\begin{tabular}{lcc}
\toprule
\textbf{Batch Type} & \textbf{Rewrite Acc.} & \textbf{Locality} \\
\midrule
Homogeneous (ours) & 23.1\% & 42.5\% \\
Heterogeneous & 43.5\% & 29.4\% \\
\bottomrule
\end{tabular}
\end{table}

Regarding hyperparameter analysis in Figure \ref{fig:heatmap}, we observe distinct patterns:
low thresholds fail to filter low-quality samples, limiting corrective opportunities; The total number of edits is limited, and the filtered erroneous samples cannot receive sufficient corrective training, limiting overall performance. Many erroneous samples in the early stage undergo continuous learning, causing the model to quickly fall into local optima and catastrophic degradation of generalization. Later learning yields little improvement. In the upper-right quadrant, the absence of error feedback leaves many suboptimal samples, and the model editing efficiency is relatively high, approximating an open-loop editing process. In the lower-right quadrant, the model training efficiency is the lowest, but excessive editing can introduce overfitting risks, wasting computational resources on edits with low marginal utility. We also evaluate threshold sensitivity (see \ref{app:thresh_sensitivity}).

\begin{figure}[htbp]
\centering
\subfloat[]{
\includegraphics[width=0.48\textwidth]{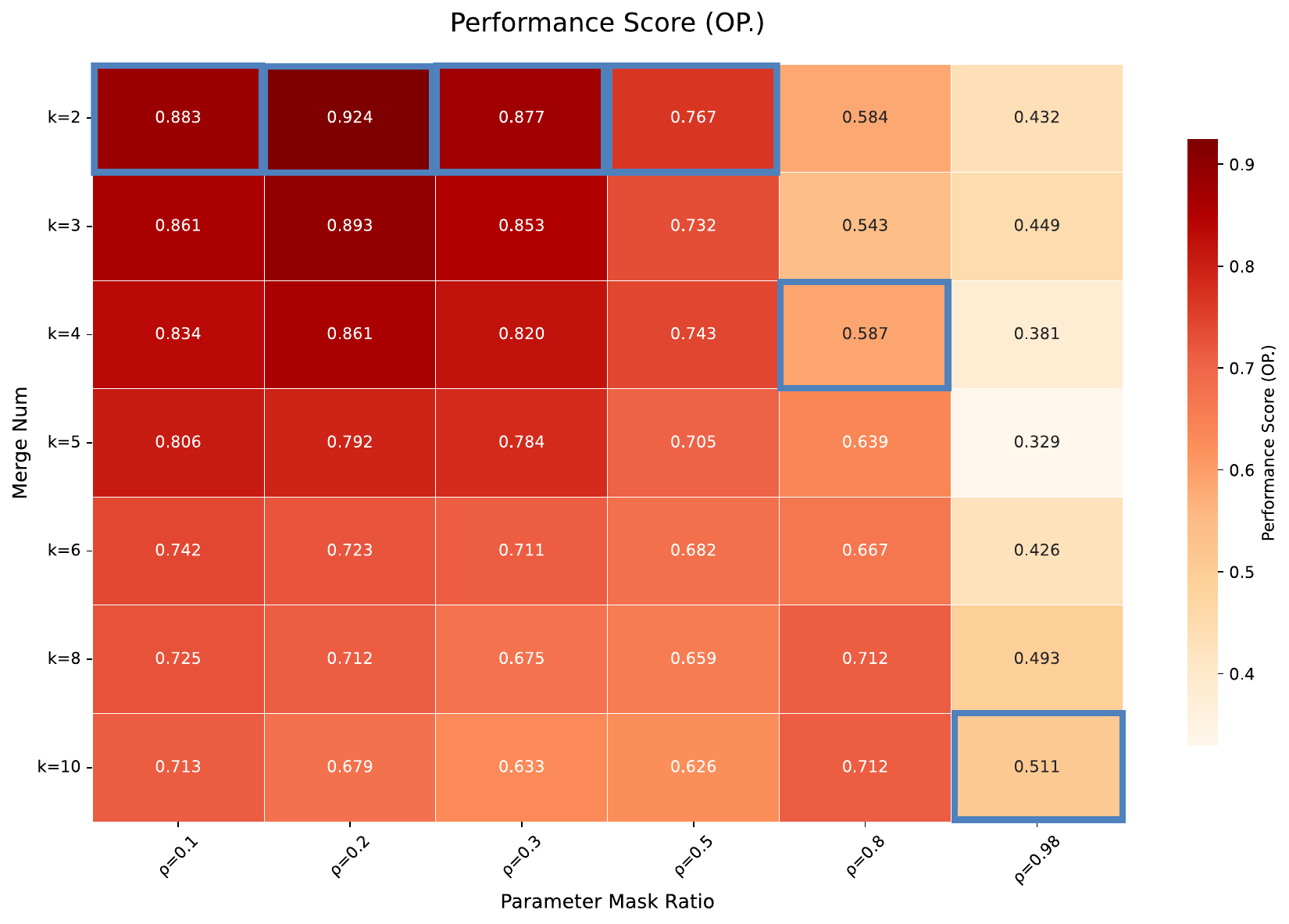}
\label{fig:heatmap_sub1}
}
\hfill
\subfloat[]{
\includegraphics[width=0.48\textwidth]{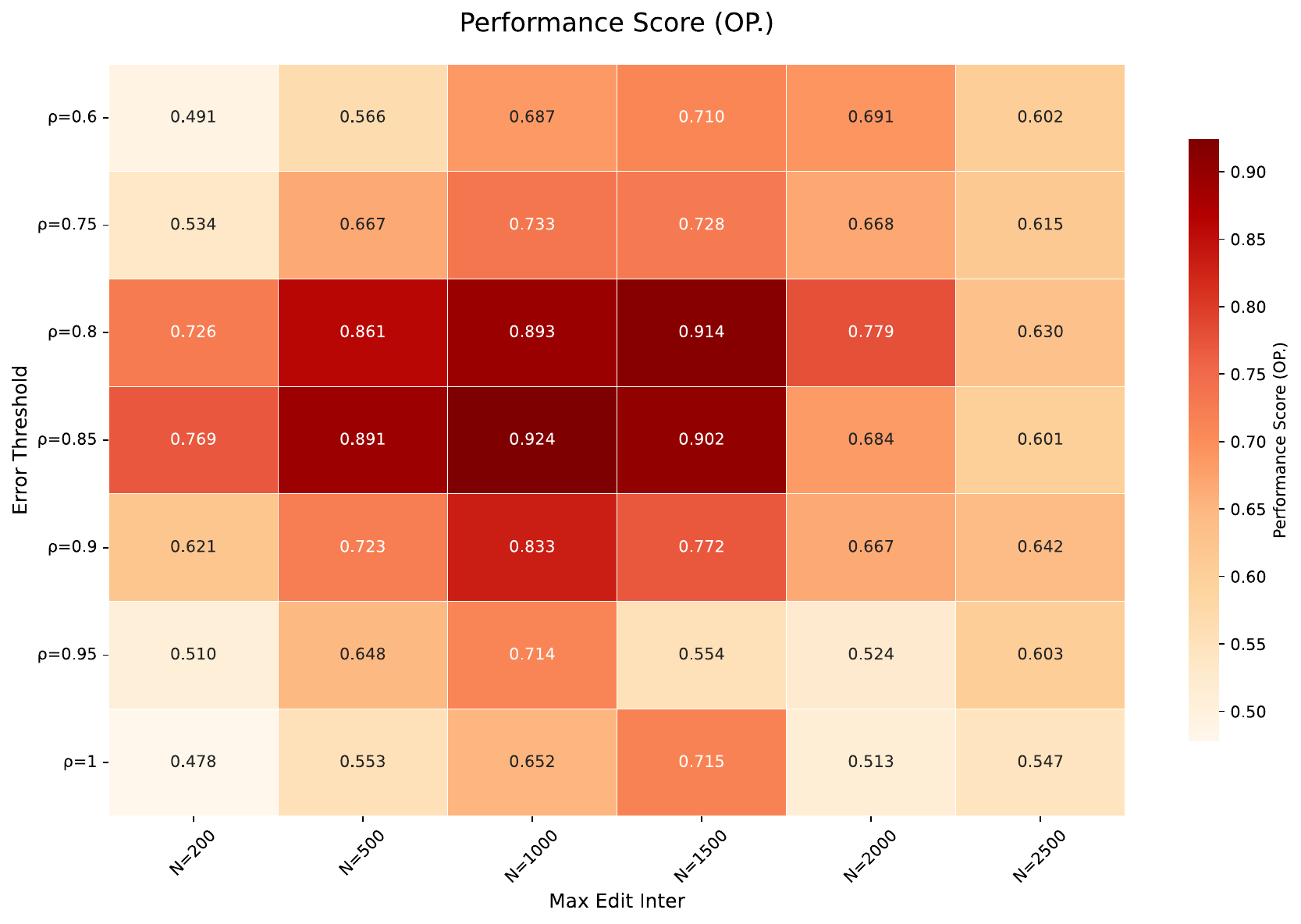}
\label{fig:heatmap_sub2}
}
\caption{\textbf{Performance heatmap for the N=120 QA task on the LLaMA3 model.} Figure (a) shows the sensitivity analysis of two hyperparameters: the number of subspaces and the amount of updated parameters;
Figure (b) analyzes the impact of error threshold and maximum iteration count on performance, with optimal performance observed at intermediate values.}
\label{fig:heatmap}
\end{figure}

\section{Conclusion}

In this work, we proposed \textbf{REPAIR}, a robust framework for lifelong model editing integrating error closed-loop feedback, inner-batch knowledge distillation, and loss-aware subspaces merging. Extensive experiments demonstrate that REPAIR maintains high performance under small-scale edits and exhibits remarkable robustness in large-scale editing scenarios, consistently outperforming existing baselines. These results highlight the potential of combining memory-aware strategies with optimization-driven editing for reliable and precise model updates. The intra-group distillation explicitly encourages feature alignment among similar samples, guiding the elimination and recombination of inconsistent samples. The loss-aware merging assigns higher weights to subspaces achieving lower training loss, effectively preserving reliable knowledge and reducing information dilution. Extensive experiments show that REPAIR consistently improves reliability and generalization, and demonstrates clear advantages in large-scale editing scenarios, highlighting the effectiveness of coordinated sample-level alignment and global reliability-aware merging.

\section*{Impact Statement}
This work studies lifelong model editing for LLMs, aiming to enable efficient, targeted updates (e.g., correcting outdated or incorrect facts) while reducing unintended side effects on unrelated behaviors. Positive impacts include lower computational cost compared to retraining, faster correction of factual errors, and improved maintainability of deployed systems in dynamic knowledge environments.
However, model editing methods may also introduce risks. First, they can be misused to inject misinformation, bias, or malicious instructions into a deployed model, especially if editing interfaces are exposed without appropriate access control and auditing. Second, imperfect locality could unintentionally alter model behavior on non-target content, potentially affecting downstream users unevenly across domains or demographic groups.
REPAIR is designed to mitigate some of these concerns through locality-oriented routing, closed-loop monitoring, and error-driven pruning, which help detect and correct harmful spillovers during sequential updates. We emphasize that safe deployment still requires governance beyond the algorithm itself, such as authenticated edit workflows, logging and review of edits, and evaluation suites that test both the intended changes and potential side effects.
% In the unusual situation where you want a paper to appear in the
% references without citing it in the main text, use \nocite
\nocite{langley00}

\bibliographystyle{icml2026}
\bibliography{example_paper}

%%%%%%%%%%%%%%%%%%%%%%%%%%%%%%%%%%%%%%%%%%%%%%%%%%%%%%%%%%%%%%%%%%%%%%%%%%%%%%%
%%%%%%%%%%%%%%%%%%%%%%%%%%%%%%%%%%%%%%%%%%%%%%%%%%%%%%%%%%%%%%%%%%%%%%%%%%%%%%%
% APPENDIX
%%%%%%%%%%%%%%%%%%%%%%%%%%%%%%%%%%%%%%%%%%%%%%%%%%%%%%%%%%%%%%%%%%%%%%%%%%%%%%%
%%%%%%%%%%%%%%%%%%%%%%%%%%%%%%%%%%%%%%%%%%%%%%%%%%%%%%%%%%%%%%%%%%%%%%%%%%%%%%%
\newpage
\appendix
\onecolumn
\section{Statement}
\subsection{Ethics Statement}
This work studies safe, auditable editing of large language models using only publicly available datasets (ZsRE, WikiBigEdit, and a hallucination set) and off‑the‑shelf pretrained models; no human subjects or personally identifiable data were collected. We follow all dataset/model licenses and the double‑blind review policy. Potential risks include misuse of editing to inject misinformation or to weaken safety constraints, and unintended spillover of edits to unrelated behaviors. To mitigate these risks, our framework emphasizes locality and closed‑loop error checks before and after integration, and we report reliability–generalization–and locality metrics to surface side effects. Upon release, we will include guardrails such as edit logs, validation suites, reversible edits, and instructions for responsible use. These design choices align with REPAIR’s stated goal of precise updates with locality safeguards.
\subsection{Reproducibility Statement}
We will release our code, configs, and seeds to reproduce all results end‑to‑end after acceptance.  Scripts fetch data/models, fix environments, and regenerate all tables/figures with the exact metrics (Rel./Gen./Loc./OP., PPL); hardware and hyperparameters are documented.
\subsection{AI Usage Statement}
We used large language model–based tools during writing and implementation for text polishing, grammar and usage checks, and programming assistance (e.g., example code, refactoring, comments, and script templates). All AI-generated suggestions were reviewed, revised, and validated by the authors. The experimental design, data processing, result analysis, and conclusions were conducted independently by the authors; AI tools do not constitute authorship or academic credit. No sensitive or restricted data were provided to the tools, and they were not used to automatically generate experimental results or to replace essential human judgment.
\section{Related Work}
\label{gen_inst}
\subsection{Continual Learning}
Continual Learning (CL)—also known as Incremental Learning or Lifelong Learning—aims to enable models to learn sequentially from a stream of tasks without forgetting previously acquired knowledge. The core challenge in CL is catastrophic forgetting, where adapting to new tasks leads to a significant degradation in performance on earlier tasks \cite{Kirkpatrick2017EWC,McCloskey1989Catastrophic}. To address this, numerous methods have been proposed, which can be broadly categorized into five groups: regularization-based, replay-based, optimization-based, representation-based, and architecture-based approaches.

Regularization-based methods mitigate forgetting by adding constraints to the loss function to preserve important parameters or behaviors from previous tasks. For example, Elastic Weight Consolidation (EWC) leverages Fisher information to regularize parameter updates \cite{Kirkpatrick2017EWC}, while Learning without Forgetting (LwF) uses knowledge distillation to maintain output consistency \cite{Li2018LwF}.

Replay-based methods retain or generate samples from previous tasks to approximate old data distributions. Experience replay stores a subset of prior samples in a memory buffer \cite{LopezPaz2017GEM}, whereas generative replay synthesizes pseudo-samples using deep generative models such as GANs or VAEs \cite{Shin2017DGR}.

Optimization-based methods manipulate the optimization process itself to avoid interference between tasks. Gradient Episodic Memory (GEM) projects gradients so as not to increase loss on previous tasks \cite{LopezPaz2017GEM}, while Orthogonal Gradient Descent (OGD) promotes updates that are orthogonal to gradient directions associated with past tasks \cite{Farajtabar2020OGD}.

Representation-based methods focus on learning robust and transferable features that are less prone to forgetting. Self-supervised learning \cite{Fini2022SSLCL} and large-scale pre-training \cite{Mehta2023PretrainingLL} have been shown to bolster CL performance by providing more stable representations.

Architecture-based methods \cite{Ran2024BraininspiredCP,Rusu2016Progressive,Mallya2018PackNet}dynamically expand or partition the network to allocate task-specific parameters. Progressive Networks add new columns for each incoming task with lateral connections to prior columns \cite{Rusu2016Progressive}, while PackNet iteratively prunes and reuses weights to free capacity for new tasks \cite{Mallya2018PackNet}.

Recent trends extend CL to more realistic and challenging settings, including class-incremental learning (CIL), task-free CL (TFCL), online CL (OCL), and applications across object detection, semantic segmentation, reinforcement learning, and natural language processing \cite{Wang2023CLSurvey}.

\subsection{Model editing}
Model editing targets post-hoc modification of a trained model’s behavior to insert, correct, or remove specific knowledge, ideally without harming unrelated capabilities. A common taxonomy distinguishes (i) \emph{direct / training-free} parameter edits, (ii) \emph{learning-based} editors that predict weight updates, and (iii) \emph{semi-parametric} systems that externalize edits via retrieval or memory; recent surveys consolidate definitions, benchmarks, and open challenges \cite{Wang2024KESurvey}.

ROME locates causal mediators of factual associations in mid-layer feed-forward (MLP) modules of Transformers and applies a rank-one update to edit a single fact \cite{Meng2022ROME}. MEMIT extends this idea to \emph{mass editing}, deriving multi-layer closed-form updates that scale to thousands of edits in large models while maintaining stronger locality than prior methods \cite{Meng2023MEMIT}. Although effective, subsequent analyses highlight stability issues under \emph{sequential} edits and propose remedies \cite{Gupta2024RebuildROME}.

Early work framed editing as learning a small hypernetwork to predict weight deltas from an edit specification: KnowledgeEditor (KE) learns constrained updates to change a model’s factual prediction while preserving behavior on paraphrases \cite{DeCao2021KE}. MEND trains lightweight editor networks to transform fine-tuning gradients, enabling fast, local edits at scale across architectures \cite{Mitchell2022MEND}. Instruction-driven variants further condition edits on natural-language instructions to improve usability and control \cite{Zhang2024InstructEdit}.

Semi-parametric approaches such as SERAC store edits in an external key–value memory and learn to route between the base model and retrieved counterfactuals, achieving strong reliability and specificity without permanently altering base parameters \cite{Mitchell2022SERAC}. This design is attractive when edits must be audited, reverted, or scoped to contexts.

Editing methods are typically assessed along \emph{reliability} (does the change take effect), \emph{locality/specificity} (does unrelated behavior remain intact), and \emph{generalization} (do edits transfer to paraphrases and contexts). Standard benchmarks include CounterFact and zsRE \cite{CounterFact2022,Levy2017ZsRE}. Recent studies examine \emph{ripple effects} beyond targeted facts, revealing broader side impacts on reasoning and distributed knowledge, and call for more rigorous, stress-testing evaluations \cite{Cohen2024Ripple}. Overall, direct, learning-based, and semi-parametric approaches offer complementary trade-offs in edit scalability, controllability, and safety; combining precise localization with guardrails (e.g., retrieval gating, edit scopes, or validation filters) remains an active direction \cite{Wang2024KESurvey}.

\noindent\textbf{Relation to MoE/LoRA-style lifelong editors.} Recent lifelong editing approaches often instantiate a mixture of experts (MoE) or Mixture-of-LoRA design, where new edits are stored as isolated experts/adaptors and selected by a router. In contrast, REPAIR emphasizes \emph{closed-loop} correction (monitoring, pruning, and reintegration) and \emph{loss-aware merging} that periodically fuses shards into a consolidated parameter subspace, aiming to mitigate unbounded expert growth and to improve long-term coherence.

\noindent\textbf{Relation to retrieval-augmented continual prompting.} Our pipeline is also related to methods that use retrieval and continuous adaptation for knowledge editing (e.g., RECIPE), but REPAIR differs in that it focuses on parametric updates with explicit shard-level health checks, pruning, and reintegration, rather than relying on a prompt-centric external retrieval interface.

\noindent\textbf{Relation to MoE/LoRA-style lifelong editors.} Recent lifelong editing approaches often instantiate a mixture of experts (MoE) or Mixture-of-LoRA design, where new edits are stored as isolated experts/adaptors and selected by a router. In contrast, REPAIR emphasizes \emph{closed-loop} correction (monitoring, pruning, and reintegration) and \emph{loss-aware merging} that periodically fuses shards into a consolidated parameter subspace, aiming to mitigate unbounded expert growth and to improve long-term coherence.

\noindent\textbf{Relation to retrieval-augmented continual prompting.} Our pipeline is also related to methods that use retrieval and continuous adaptation for knowledge editing (e.g., RECIPE), but REPAIR differs in that it focuses on parametric updates with explicit shard-level health checks, pruning, and reintegration, rather than relying on a prompt-centric external retrieval interface.

\section{Experiment details}
\label{app:experiment_details}
The experiment details are given in Table \ref{Dataset statistics}, and hyperparameters are in Table \ref{Hyperparameter}.

\subsection{Feature similarity for batching}
\label{app:feat_sim}
We examine whether the feature representations used to construct homogeneous batches capture meaningful similarity. To validate the routing criterion, we quantitatively analyze the feature space across different editing tasks (ZsRE vs. Hallucination).

As shown in Table~\ref{tab:feat_sim}, intra-task self-similarity is substantially higher than cross-task similarity, indicating that the features capture task-specific semantics.

\begin{table}[t]
\centering
\caption{Mean cosine similarity of feature representations across editing tasks.}
\label{tab:feat_sim}
\small
\setlength{\tabcolsep}{10pt}
\begin{tabular}{lc}
\toprule
\textbf{Comparison} & \textbf{Mean Cosine Similarity} \\
\midrule
ZsRE (Self-Similarity) & 0.6824 \\
Hallucination (Self-Similarity) & 0.5878 \\
Cross-Task (ZsRE vs. Hallucination) & 0.4423 \\
\bottomrule
\end{tabular}
\end{table}

\subsection{Threshold sensitivity}
\label{app:thresh_sensitivity}
To further address the concern that REPAIR may be sensitive to the error-filtering threshold, we conduct a one-dimensional sweep over the filtering threshold $\gamma_{\text{filter}}$ while keeping other settings fixed. Table~\ref{tab:thresh_sensitivity} reports Rewrite Accuracy (edit success on the target prompt) and Locality Accuracy.

\begin{table}[t]
\centering
\caption{\textbf{Sensitivity to the error-filtering threshold} $\gamma_{\text{filter}}$ (\%).}
\label{tab:thresh_sensitivity}
\resizebox{\columnwidth}{!}{%
\begin{tabular}{c|ccccccccc}
\toprule
$\gamma_{\text{filter}}$ & 0.1 & 0.2 & 0.3 & 0.4 & 0.5 & 0.6 & 0.7 & 0.8 & 0.9 \\
\midrule
Rewrite Acc.  & 54.85 & 54.85 & 54.85 & 54.85 & 54.85 & 54.85 & 54.85 & 54.85 & 54.85 \\
Locality Acc. & 54.18 & 54.18 & 72.43 & 94.69 & 99.44 & 99.44 & 99.44 & 100.00 & 100.00 \\
\bottomrule
\end{tabular}%
}
\end{table}

As shown in the sweep, REPAIR exhibits a broad stable operating range ($\gamma_{\text{filter}}\in[0.5,0.9]$), in which Locality Accuracy consistently reaches near-perfect levels (\(>99\%\)) while Rewrite Accuracy remains strictly stable (0.5485), indicating no trade-off degradation. In contrast, extremely low thresholds ($\gamma_{\text{filter}}<0.3$) reduce locality because insufficient filtering allows interference to propagate. Overall, the method does not require precise fine-tuning of $\gamma_{\text{filter}}$: any value chosen in the upper spectrum yields optimal behavior, supporting that REPAIR is robust rather than hyperparameter-sensitive.

\begin{table}[t]
\caption{Dataset statistics}
\label{Dataset statistics}
\begin{center}
\small
\begin{tabular}{ccccc}
\toprule
Task& Editing Data&N&Pre-edit(LLaMA/Qwen)&Locality Data\\
\midrule
\multirow{2}{*}{QA}  &ZsRE&1000      &0.25/0.21 ACC &NQ \cite{2019NQdata} \\
&WikiBigEdit &500K&0.36/0.32 ACC &NQ\\
Hallu. &SelfCheckGPT&600&28.7/29.1 PPL&RedPajama \cite{DBLP:conf/nips/WeberFAOAALNYAA24}\\
\bottomrule
\end{tabular}
\end{center}
\end{table}

\begin{table}[t]
\caption{Hyperparameter settings}
\label{Hyperparameter}
\begin{center}
\small
\begin{tabular}{cccccc}
\toprule
\multicolumn{6}{c}{ZsRE on LLaMA-3}\\
\cmidrule{1-6}
\multicolumn{1}{c}{HYPER}  &\multicolumn{1}{c}{VALUE}  &\multicolumn{1}{c}{HYPER}  &\multicolumn{1}{c}{VALUE}  &\multicolumn{1}{c}{HYPER}  &\multicolumn{1}{c}{VALUE}\\
\cmidrule{1-6}
Mask ratio   &0.20 &Edit\_lr      &0.90 &Err\_Thresh  &0.85\\
$\lambda_a$   &1.00&$\lambda_{KD}$ &1.00 &Max\_iter &10000\\
Temperature  &2.00  &Act ratio  &0.20 &Layer\_ID &29.00\\
$\gamma_1$ &2.00  & $\gamma_2$& 20.00&$\gamma$&10.00 \\
$n_{\mathrm{iter}}$ &30.00 &$\lambda$&0.20 &Act\_ratio &0.30\\
\bottomrule
\multicolumn{6}{c}{ZsRE on Qwen2.5}\\
\cmidrule{1-6}
Mask ratio   &0.20 &Edit\_{lr}     &0.90 &Err\_Thresh  &0.85\\
$\lambda_a$   &2.00 &$\lambda_{KD}$ &1.00 &Max\_iter &10000\\
Temperature  &2.00  &Act ratio  &0.88 &Layer\_ID &23.00\\
$\gamma_1$ &5.00  & $\gamma_2$& 20.00&$\gamma$&10.0 \\
$n_{\mathrm{iter}}$ &50.00 &$\lambda$&0.30 &Act\_ratio &0.30\\
\bottomrule
\multicolumn{6}{c}{Selfcheck GPT on LLaMA-3-8B}\\
\cmidrule{1-6}
Mask ratio   &0.20 &Edit\_{lr}     &1.00 &Err\_Thresh  &0.85\\
$\lambda_a$   &5.00 &$\lambda_{KD}$ &1.00 &Max\_iter &5000\\
Temperature  &2.00  &Act ratio  &0.88 &Layer\_ID &27.00\\
$\gamma_1$ &5.00  & $\gamma_2$& 20.00&$\gamma$&10.00 \\
$n_{\mathrm{iter}}$ &50.00 &$\lambda$&0.20 &Act\_ratio &0.80\\
\bottomrule
\end{tabular}
\end{center}
\end{table}

% & \multicolumn{15}{c}{Deepseek R1-1.5b} \\

\begin{table}[h]
\caption{Main results for QA on DeepSeek‑R1‑1.5B $N$: Num Edits.}
\label{Tab:DP}
\centering
\small
\setlength{\tabcolsep}{2pt}
\begin{tabular}{lccc|c|ccc|c|ccc|c|ccc|c|}
\toprule
\multirow{2}{*}{\textbf{Method}} & \multicolumn{4}{c}{$N = 1$} & \multicolumn{4}{c}{$N = 30$} & \multicolumn{4}{c}{$N = 120$} & \multicolumn{4}{c}{$N = 1000$} \\
\cmidrule{2-17}

& Rel. & Gen. & Loc. & OP. & Rel. & Gen. & Loc. & OP. & Rel. & Gen. & Loc. & OP. & Rel. & Gen. & Loc. & OP. \\
\midrule
& \multicolumn{15}{c}{DeepSeek‑R1‑1.5B (ZsRE)} \\
\cmidrule{1-17}
FT-L & 0.43 & 0.42 & 0.95 & 0.56 & 0.32 & 0.33 & 0.46 & 0.36 & 0.21 & 0.21 & 0.15 & 0.19 & 0.17 & 0.15 & 0.09 & 0.13 \\
FT-EWC & 0.97 & \textbf{0.94} & 0.15 & 0.52 & 0.82 & 0.81 & 0.02 & 0.24 & 0.63 &0.64& 0.02 & 0.20 & 0.57 &0.56 & 0.02 & 0.19 \\
MEND & 0.95 & \textbf{0.94} & 0.98 & 0.96 & 0.42 & 0.42 & 0.18 & 0.32 & 0.18 & 0.12 & 0.07 & 0.11 & 0.8 & 0.03 & 0.00 & 0.00 \\
ROME & 0.87 & 0.87 & 0.99 & 0.91 & 0.66 & 0.64 & 0.72 & 0.67 & 0.17 & 0.18 & 0.09 & 0.14 & 0.01 & 0.01 & 0.01 & 0.01 \\
MEMIT-M & 0.88 & 0.87 & 0.99 & 0.91 & 0.71 & 0.72 & 0.92 & 0.78 & 0.63 & 0.65 & 0.78 & 0.68 & 0.48 & 0.47 & 0.53 & 0.49 \\
DEFER & 0.62 & 0.60 & 0.82 & 0.67 & 0.58 & 0.57 & 0.57 & 0.57 & 0.34 & 0.31 & 0.23 & 0.29 & 0.07 & 0.06 & 0.02 & 0.04 \\
GRACE & \textbf{0.98} & 0.31 & 0.99 & 0.67 & \textbf{0.92} & 0.22 & \textbf{0.98} & 0.58 & \textbf{0.89} & 0.13 & \textbf{1.00} & 0.49 & \textbf{0.83} & 0.05 & \textbf{0.94} & 0.34 \\
WISE & 0.92 & 0.90 & \textbf{1.00} & 0.94 & 0.86 & 0.85 & 0.92 & 0.88 & 0.72 & 0.72 & 0.87&\textbf{0.77}& 0.49 & 0.47 & 0.47 & 0.48\\
\cmidrule{1-17}
\rowcolor{lightpurple}
\textbf{REPAIR} & 0.93 & 0.93 & \textbf{1.00} & \textbf{0.95} & 0.91 & \textbf{0.89} & 0.87 & \textbf{0.89$\uparrow$} & 0.74 & \textbf{0.74} & 0.82 & \textbf{0.77$\uparrow$} & 0.59 &\textbf{0.57} & 0.61 & \textbf{0.59$\uparrow$} \\
\bottomrule
\end{tabular}
\end{table}

\begin{table}[h]
\caption{Main results for QA (ZsRE) on multi-model editing with error distribution.}
\label{Tab:with-err}
\centering
\footnotesize
\setlength{\tabcolsep}{2pt}
\begin{tabular}{lcccccc}
\toprule
\multirow{2}{*}{\textbf{Method}} & \multicolumn{3}{c}{$N = 1$} & \multicolumn{3}{c}{$N = 30$} \\
\cmidrule{2-7}
& Rel. & Gen. & Loc. & Rel. & Gen. & Loc.  \\
% \midrule
% \multicolumn{7}{c}{ZsRE} \\
\cmidrule{1-7}
LLaMA-3-8B & $0.94\pm0.008$ & $0.92\pm0.01$ & $1.00_{-0.02}^{+0.00}$ & $0.93\pm0.003$ &$0.90\pm0.003$ & $0.87\pm0.004$ \\
Qwen2.5-7B & $0.98\pm0.02$ & $0.95\pm0.03$ & $1.00_{-0.02}^{+0.00}$ & $0.93\pm0.04$ &$0.90\pm0.03$  & $0.93\pm0.01$ \\
DeepSeek-R1&$0.93\pm0.02$&$0.92\pm0.03$&$0.99\pm0.01$&$0.91\pm0.01$&$0.89\pm0.03$&$0.87\pm0.01$\\
GPT2-XL&$0.91\pm0.03$&$0.92\pm0.03$&$0.99\pm0.01$&$0.88\pm0.03$&$0.88\pm0.02$&$0.84\pm0.01$\\
\cmidrule{1-7}
\multirow{2}{*}{\textbf{Method}} & \multicolumn{3}{c}{$N = 120$} & \multicolumn{3}{c}{$N = 1000$} \\
\cmidrule{2-7}
& Rel. & Gen. & Loc. & Rel. & Gen. & Loc. \\
\midrule
\cmidrule{1-7}
LLaMA-3-8B  &$0.76\pm0.03$&$0.74\pm0.02$&$1.00_{-0.04}^{+0.00}$&$0.68\pm0.05$&$0.65\pm0.01$&$0.89\pm0.04$ \\
Qwen2.5-7B & $0.81\pm0.04$&$0.80\pm0.05$&$0.92\pm0.03$ &$0.72\pm0.05$&$0.70\pm0.04$&$0.67\pm0.03$\\
DeepSeek-R1&$0.74\pm0.03$&$0.74\pm0.04 $&$0.82\pm0.05$&$0.59\pm0.02$&$0.57\pm0.01$&$0.61\pm0.03$\\
GPT2-XL&$0.79\pm0.02$&$0.77\pm0.01 $&$0.80\pm0.03$&$0.61\pm0.03$&$0.62\pm0.01$&$0.68\pm0.02$\\
\bottomrule
\end{tabular}
\end{table}

\begin{figure}
\centering
\includegraphics[width=0.8\linewidth]{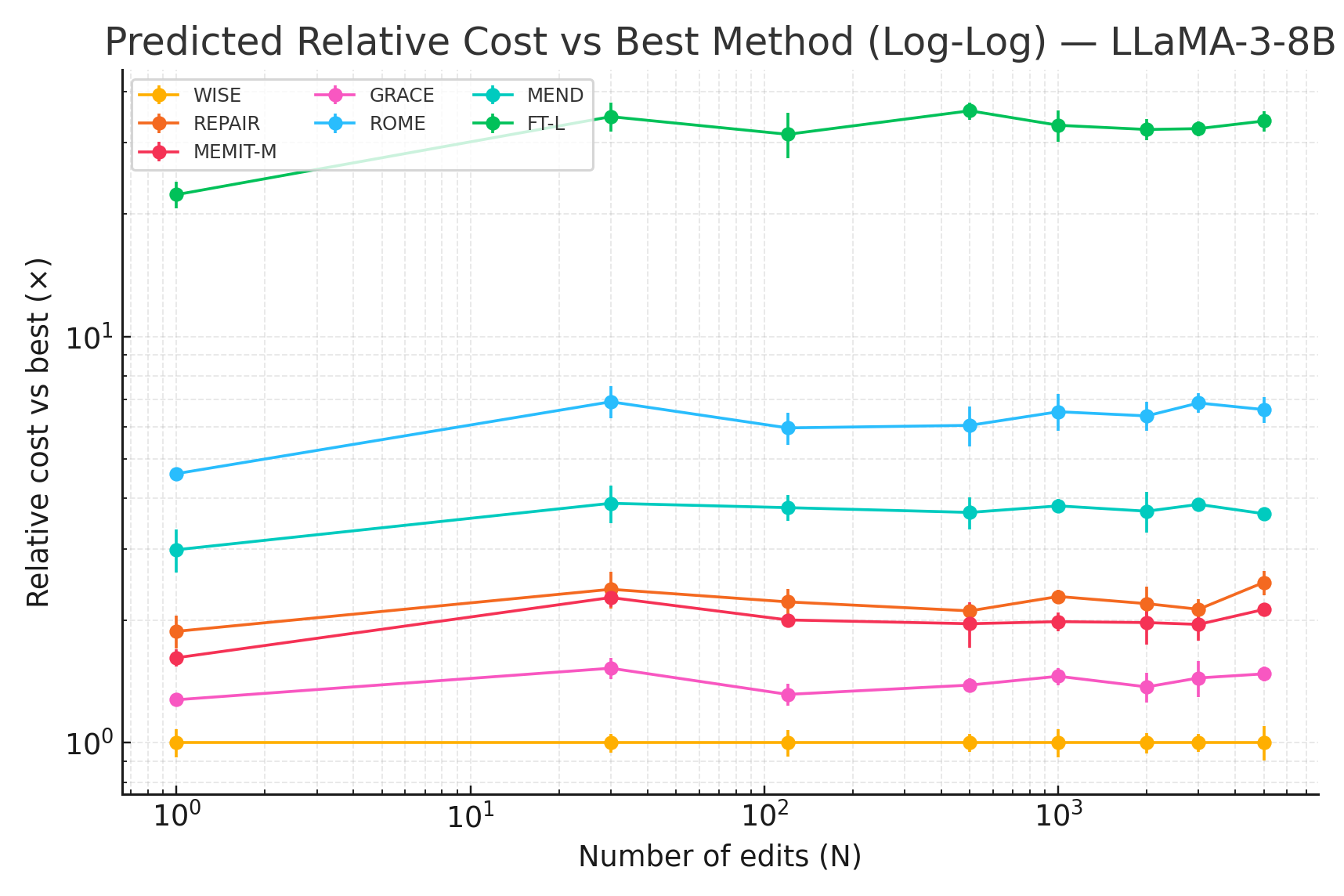}
\caption{\textbf{Cost-Performance Assessment.} The total runtime of each method scales approximately linearly with the editing scale \(N\), appearing as straight lines with slopes close to 1 in log-log coordinates. This indicates that the primary overhead is proportional to the number of edited entries.}
\label{fig:timeconsuming}
\end{figure}

Under identical hardware and batch configurations, WISE typically has lower per-edit overhead. REPAIR exhibits a similar scaling slope but may incur a higher constant-factor cost, mainly due to:
\begin{itemize}
\item Distribution-aware clustering and reorganization;
\item Additional forward/backward passes for in-batch distillation;
\item Closed-loop monitoring and the (occasionally triggered) pruning/retraining procedure;
\item The final \textbf{merging (TIES)} cost.
\end{itemize}

To quantify the practical overhead, we additionally profile wall-clock time and resource usage against WISE under a small-scale sequential editing setting (\(N=100\)). As shown in Table~\ref{tab:wise_repair_cost_n100}, REPAIR introduces a modest initialization overhead, while the overall editing time and memory footprint remain comparable, and inference latency is slightly improved.

\begin{table}[t]
\centering
\small
\setlength{\tabcolsep}{6pt}
\caption{\textbf{Time and resource comparison with WISE} in a sequential editing setup (\(N=100\)).}
\label{tab:wise_repair_cost_n100}
\begin{tabular}{lcc}
\toprule
\textbf{Metric} & \textbf{WISE} & \textbf{REPAIR} \\
\midrule
Initialization time (s) & 3.86 & 4.22 \\
Total editing time (s) & 412.96 & 411.94 \\
Avg. inference latency (ms) & 1135.32 & 1067.86 \\
Memory increase (MB) & 25118 & 25008 \\
\bottomrule
\end{tabular}
\end{table}

We emphasize that REPAIR's additional constant factors are primarily activated when closed-loop feedback detects drift (e.g., via retriggering and reintegration). Thus, the overhead is workload-dependent and can become more pronounced as the edit stream grows and more corrective cycles are triggered.

\textbf{Throughput.} Under our profiling setup, REPAIR achieves approximately $0.8$--$0.9$ edits/min at scale, while WISE reaches around $1.8$ edits/min. REPAIR trades throughput for robustness due to the added monitoring and reintegration steps.

Error bars (standard deviation across runs) indicate that REPAIR may exhibit slightly higher variance than WISE, attributable to fluctuations in retriggering frequency and sample distribution characteristics.

\section{Theoretical Analysis and Proof Sketches}

We now provide theoretical justifications for the stability and convergence of the proposed REPAIR framework. We introduce formal assumptions and derive lemmas and theorems that characterize the behavior of our method.

\subsection{Preliminaries}

\begin{assumption}[Standard Optimization Setting]
We assume that the loss function $\mathcal{L}(\Theta)$ is $L$-smooth, i.e.,
\[
\|\nabla \mathcal{L}(\Theta_1) - \nabla \mathcal{L}(\Theta_2)\| \le L \|\Theta_1 - \Theta_2\|,
\]
and bounded below by $\mathcal{L}^*>-\infty$. Learning rates satisfy $\eta_t>0$ and $\sum_t \eta_t = \infty$, $\sum_t \eta_t^2 < \infty$.
\end{assumption}

\subsection{Stability of Masked Gradient Updates}

\begin{lemma}[Norm Bound under Masked Updates]
\label{Stability of Masked Gradient Updates}
Let $g_i=\nabla_{W_{v,i}'}\mathcal{L}$ and $M_i$ be a Bernoulli mask. Then the masked update
\[
\Delta W_{v,i}' = -\eta (M_i \odot g_i)
\]
satisfies $\|\Delta W_{v,i}'\|_2 \le \eta \|g_i\|_2$.
\end{lemma}

\begin{proof}
Since $M_i$ is a coordinate projection, $M_i \odot g_i$ removes certain entries of $g_i$ and never increases its magnitude. Hence $\|M_i \odot g_i\|_2 \le \|g_i\|_2$. Multiplying by $\eta$ yields the claim.
\end{proof}

\begin{theorem}[Inter-Shard Stability]
\label{Inter-Shard Stability}
Assume masks $\{M_i\}$ are sampled independently with overlap probability $\rho^2$. Then in expectation,
\[
\mathbb{E}[\langle M_i \odot g_i, M_j \odot g_j \rangle] = \rho^2 \langle g_i, g_j \rangle.
\]
Thus, masking reduces the expected conflict between gradients of different shards.
\end{theorem}

\begin{proof}
For each coordinate $p$, $\Pr[M_i(p)=1, M_j(p)=1] = \rho^2$. Therefore, the expected inner product between masked gradients is $\rho^2$ times the original inner product. This reduces cross-shard interference and improves stability.
\end{proof}

\subsection{Closed-Loop Re-trigger Analysis}

\begin{assumption}[Error Reduction per Re-trigger]
Suppose that each re-trigger reduces the error rate of shard $i$ by at least a fixed constant $\delta>0$, unless it is already below the pruning threshold $\tau_{\mathrm{prune}}$.
\end{assumption}

\begin{lemma}[Linear Error Decrease]
Let $r_i^{(n)}$ denote the error rate after $n$ re-triggers. Under Assumption 2,
\[
r_i^{(n)} \le r_i^{(0)} - n\delta.
\]
\end{lemma}

\begin{theorem}[Finite-Time Convergence]
\label{Finite-Time Convergence}
If $r_i^{(0)}$ is the initial error rate, then after at most
\[
N \ge \frac{r_i^{(0)} - \tau_{\mathrm{prune}}}{\delta}
\]
re-triggers, the error rate satisfies $r_i^{(N)} \le \tau_{\mathrm{prune}}$.
\end{theorem}

\begin{proof}
By Lemma 3, $r_i^{(N)} \le r_i^{(0)} - N\delta$. Choosing $N$ such that $r_i^{(0)} - N\delta \le \tau_{\mathrm{prune}}$ ensures convergence below threshold in finite time.
\end{proof}

\subsection{Overall Convergence Intuition}

\begin{theorem}[Closed-Loop Stability of REPAIR]
Under Assumptions 1 and 2, the iterative process combining masked updates, inner-batch distillation, and closed-loop re-trigger forms a contractive mapping in expectation. Consequently, the system converges to a stable edited state with a bounded error rate and without catastrophic forgetting.
\end{theorem}

\begin{proof}[Proof Sketch]
Masked updates reduce the variance of parameter updates, inner-batch distillation aligns outputs across samples, and re-trigger guarantees finite-time reduction of shard-level error rates. Together, these components yield monotone improvement. By standard stochastic contraction arguments, the process converges to a fixed point characterized by consistent batch predictions and an error rate below $\tau_{\mathrm{prune}}$.
\end{proof}

% --- Features and losses on the product sphere ---
% Students: o_1,...,o_m \in S^{d-1}; teacher: u \in S^{d-1}, m=b-1.
% Cosine and variance terms (with unit-norm features):
%   L_cos = 1 - (1/m) \sum_{i=1}^m \langle o_i, u \rangle
%   L_var = (1/m) \sum_{i=1}^m \|o_i - \mu\|^2,  where  \mu = \frac{1}{m}\sum_{i=1}^m o_i
% KD objective: L_KD = \lambda L_cos + \vartheta L_var,  \lambda,\vartheta>0.

\begin{lemma}[Zero-variance at any global minimizer]\label{lem:zero-var-sphere}
Let $\mu=\frac{1}{m}\sum_{i=1}^m o_i$ and $\mathcal{L}_{\mathrm{var}}=\frac{1}{m}\sum_i\|o_i-\mu\|^2$.
If not all $o_i$ are equal, then $\mathcal{L}_{\mathrm{var}}>0$, while if $o_1=\cdots=o_m=v$ (with $\|v\|=1$) then $\mathcal{L}_{\mathrm{var}}=0$.
Hence every global minimizer of $\mathcal{L}_{\mathrm{KD}}$ on $(\mathbb{S}^{d-1})^{m}$ must satisfy $o_1=\cdots=o_m=:v$.
\end{lemma}

\begin{lemma}[Unique global minimizer]\label{lem:unique-minimizer-sphere}
Under the conclusion of Lemma~\ref{lem:zero-var-sphere}, minimizing
$\mathcal{L}_{\mathrm{KD}}(v)=\lambda\bigl(1-\langle v,u\rangle\bigr)$ over $\|v\|=1$
gives the unique solution $v^\star=u$. Therefore the unique global minimizer of $\mathcal{L}_{\mathrm{KD}}$ on $(\mathbb{S}^{d-1})^m$ is
$S^\star=[u,\ldots,u]$.
\end{lemma}

\begin{lemma}[Riemannian smoothness]\label{lem:riem-lipschitz}
Let $\mathcal{M}=(\mathbb{S}^{d-1})^m$ and endow each block with the canonical metric.
Then $\mathcal{L}_{\mathrm{KD}}$ is $\mathsf{L}_R$-smooth on $\mathcal{M}$ in the Riemannian sense:
there exists a constant
\[
\mathsf{L}_R \;\le\; \frac{2\lambda}{m} + \frac{4\vartheta}{m}
\]
such that for all $S,S'\in\mathcal{M}$,
$\|\operatorname{grad}\mathcal{L}_{\mathrm{KD}}(S)-\operatorname{grad}\mathcal{L}_{\mathrm{KD}}(S')\|
\le \mathsf{L}_R\,\operatorname{dist}_{\mathcal{M}}(S,S')$.
\emph{Sketch.} For each block $o_i$, $\nabla_{o_i}\mathcal{L}_{\mathrm{cos}}=-(\lambda/m)u$ (constant),
and $\nabla_{o_i}\mathcal{L}_{\mathrm{var}}=(2\vartheta/m)(o_i-\mu)$ with $\mu$ depending linearly on $\{o_j\}$.
Projecting to the tangent space by $(I-o_io_i^\top)$ and using the Lipschitzness of the projection map on $\mathbb{S}^{d-1}$ yields the bound.
\end{lemma}

\begin{theorem}[Convergence of cosine+variance KD on the sphere]\label{thm:rgd-sphere}
Consider Riemannian gradient descent on $\mathcal{M}=(\mathbb{S}^{d-1})^{m}$:
\[
o_i^{(t+1)} \;=\; R_{\,o_i^{(t)}}\!\bigl(-\eta_t\,\operatorname{grad}_{o_i}\mathcal{L}_{\mathrm{KD}}(S_t)\bigr)
\quad (i=1,\ldots,m),
\]
with the retraction $R_o(v)=(o+v)/\|o+v\|$. If the step sizes satisfy either
(a) a constant stepsize $0<\eta_t<2/\mathsf{L}_R$, or
(b) diminishing stepsizes $\sum_t \eta_t=\infty$, $\sum_t \eta_t^2<\infty$,
then:
\[
\mathcal{L}_{\mathrm{KD}}(S_t) \downarrow \mathcal{L}_{\mathrm{KD}}(S^\star),\qquad
\|\operatorname{grad}\mathcal{L}_{\mathrm{KD}}(S_t)\|\to 0,
\]
and every limit point of $\{S_t\}$ is a Riemannian critical point. By Lemma~\ref{lem:unique-minimizer-sphere},
the unique global minimizer is $S^\star=[u,\ldots,u]$; thus the sequence converges to $S^\star$.
\end{theorem}

\begin{proof}[Proof sketch]
Riemannian smoothness (Lemma~\ref{lem:riem-lipschitz}) on the compact manifold $\mathcal{M}$
ensures the standard descent lemma and monotone decrease for RGD under $0<\eta<2/\mathsf{L}_R$,
implying convergence of function values and gradients to zero. By Lemmas~\ref{lem:zero-var-sphere}–\ref{lem:unique-minimizer-sphere}, the only global minimizer is $S^\star$, hence all limit points coincide with $S^\star$.
\end{proof}

\subsection{Stability of Masked Gradient Updates}
Let $g_i=\nabla_{W_{v,i}'}\mathcal{L}\in\mathbb{R}^d$. A coordinate mask $M_i\in\{0,1\}^d$ acts by $(M_i\odot g_i)_p=M_i(p)\,g_{i,p}$.
\begin{lemma}[Norm Bound under Masked Updates]\label{lem:norm}
For any stepsize $\eta>0$, the masked update $\Delta W_{v,i}'=-\eta (M_i\odot g_i)$ satisfies
\[
\|\Delta W_{v,i}'\|_2 \le \eta\,\|g_i\|_2.
\]
\end{lemma}
\begin{proof}
Coordinate-wise, $|M_i(p)\,g_{i,p}|\le |g_{i,p}|$ because $M_i(p)\in\{0,1\}$. Hence $\|M_i\odot g_i\|_2\le \|g_i\|_2$, and multiplying by $\eta$ yields the claim.
\end{proof}
\begin{theorem}[Inter-Shard Inner-Product Scaling]\label{thm:overlap}
Suppose that for each coordinate $p$, the masks $M_i(p),M_j(p)\in\{0,1\}$ are sampled independently with
\[
\Pr[M_i(p)=1]=\Pr[M_j(p)=1]=\rho,\quad 0\le \rho \le 1,
\]
and masks are independent across coordinates and independent of $g_i,g_j$. Then, conditional on $g_i,g_j$,
\[
\mathbb{E}\!\left[\langle M_i\odot g_i,\, M_j\odot g_j\rangle \,\middle|\, g_i,g_j\right]
= \rho^2\,\langle g_i, g_j\rangle.
\]
In particular, masking scales the expected cross-shard alignment/conflict by the factor $\rho^2$.
\end{theorem}
\begin{proof}
By linearity of expectation and independence, for each coordinate $p$,
$\mathbb{E}[M_i(p)M_j(p)]=\mathbb{E}[M_i(p)]\,\mathbb{E}[M_j(p)]=\rho^2$. Summing over $p$ yields the result.
\end{proof}
\subsection{Closed-Loop Re-trigger Analysis}
\begin{assumption}[Error Reduction per Re-trigger]\label{ass:error}
Let $r_i^{(n)}$ denote the error rate of shard $i$ after $n$ re-triggers. There exists $\delta>0$ such that each re-trigger reduces error by at least $\delta$ whenever $r_i^{(n)}>\tau_{\mathrm{prune}}$.
\end{assumption}
\begin{lemma}[Piecewise-Linear Error Decrease]\label{lem:linear}
Under Assumption~\ref{ass:error}, for all $n\ge 0$,
\[
r_i^{(n)} \;\le\; \max\!\bigl\{\,\tau_{\mathrm{prune}},\; r_i^{(0)}-n\delta\,\bigr\}.
\]
\end{lemma}
\begin{proof}
If $r_i^{(k)}>\tau_{\mathrm{prune}}$, then $r_i^{(k+1)}\le r_i^{(k)}-\delta$. Once $r_i^{(k)}\le \tau_{\mathrm{prune}}$, the bound $r_i^{(n)}\le \tau_{\mathrm{prune}}$ propagates for all $n\ge k$. Unrolling gives the stated maximum form.
\end{proof}
\begin{theorem}[Finite-Time Hitting the Pruning Threshold]\label{thm:finite}
Let
\[
N_\star \;=\; \Bigl\lceil \frac{(r_i^{(0)}-\tau_{\mathrm{prune}})_+}{\delta} \Bigr\rceil
\quad\text{where } (x)_+ := \max\{x,0\}.
\]
After at most $N_\star$ re-triggers, we have $r_i^{(N_\star)}\le \tau_{\mathrm{prune}}$.
\end{theorem}
\begin{proof}
By Lemma~\ref{lem:linear}, choose the smallest integer $N_\star$ such that $r_i^{(0)}-N_\star\delta\le \tau_{\mathrm{prune}}$. Then $r_i^{(N_\star)}\le \tau_{\mathrm{prune}}$.
\end{proof}

\section{Algorithms}
The pseudocode for error feedback, network pruning, sample knowledge distillation and reintegration, and the loss-based weighted ties merge strategy is as follows:

% ====== Main Training Loop ======
\begin{algorithm}[t]
\caption{REPAIR: Closed-Loop Lifelong Model Editing (Training)}
\small
\begin{algorithmic}[1]
\Require Pretrained model $f_{\theta_0}$; target FFN value matrix $W_v$; \#shards $K$;
mask ratio $\rho$; thresholds $(\epsilon,\tau_E,\tau_{\text{prune}},\tau_{\text{correct}},\epsilon_{\text{cons}})$;
margins $(\gamma_1,\gamma_2,\gamma)$; KD weights $(\lambda,\vartheta)$ for Eq.(4);
routing-loss weight $\lambda_a$; batch size $b$; optional temperature $T$ for soft KD.
\State Initialize side memories $W'_{v,i}\gets W_v$ and masks $M_i\sim \text{Bernoulli}(\rho)$ for $i=1..K$; feedback pool $\mathcal{E}\gets\emptyset$; residual pool $\mathcal{R}\gets\emptyset$.
\For{each incoming edit triple $(x_e,y_e,x_{\text{loc}})$}
\State $i^\star \gets \Call{AssignShard}{x_e}$ \Comment{Shard assignment by activation score}
\State $\mathcal{B} \gets \Call{FormBatches}{\{x_e\}\cup \mathcal{R},~b}$ \Comment{Distribution-aware batching}
\For{each batch $B=\{x^{(0)},\ldots,x^{(b-1)}\}\in \mathcal{B}$}
\State $i \gets \Call{AssignShard}{x^{(0)}}$ \Comment{Target shard for this batch}
\State $L_{\text{edit}} \gets \Call{AutoregCE}{B}$ \Comment{Autoregressive cross-entropy}
\State $L_{\text{KD}} \gets \Call{IntraBatchKD}{B,\lambda,\vartheta,T}$ \Comment{Eq.(4); optional soft KD}
\State $L_{\text{act}} \gets \Call{RoutingMargin}{B,\gamma_1,\gamma_2,\gamma}$ \Comment{Eq.(7)}
\State $L_{\text{batch}} \gets L_{\text{edit}} + \lambda_a L_{\text{act}} + L_{\text{KD}}$
\State \Call{MaskedUpdate}{$W'_{v,i}, M_i, L_{\text{batch}}$} \Comment{$W'_{v,i}\leftarrow W'_{v,i}-\eta(M_i\odot \nabla L)$}
\State \Call{FilterAndRecluster}{$B,\epsilon_{\text{cons}},\mathcal{R}$} \Comment{Move high-$L_{\text{KD}}$ samples to residual pool}
\EndFor
\State $(\hat{y},c) \gets \Call{Evaluate}{x_e,y_e}$ \Comment{$c\in\{0,1\}$ indicates success}
\If{$c=0$} \State $\mathcal{E}\gets \mathcal{E}\cup\{(x_e,y_e)\}$ \EndIf
\If{$|\mathcal{E}|>\tau_E$ \textbf{ or } $\max_i \Call{ErrorRate}{\mathcal{E},i}>\tau_{\text{prune}}$}
\State \Call{ReTrigger}{$\mathcal{E}$} \Comment{Prune worst shard, rebuild, and retrain}
\EndIf
\EndFor
\State \Call{LossAwareTIESMerge}{$\{W'_{v,i}\}_{i=1}^K,~W_v$} \Comment{Loss-aware weighted TIES merge}
\end{algorithmic}
\end{algorithm}

% ====== Inference (Routing at test time) ======
\begin{algorithm}[htbp]
\caption{REPAIR Inference with Dual-Memory Routing}
\begin{algorithmic}[1]
\Function{RouteAndPredict}{$x$}
\State compute $a(x)\gets \Call{FFNActivation}{x}$  \Comment{Activation $a(x)$ at the target FFN layer}
\For{$i=1..K$} \State $\Delta^{(i)}_{\text{act}}(x)\gets \|\,a(x)\cdot (W'_{v,i}-W_v)\,\|_2$ \EndFor
\If{$\max_i \Delta^{(i)}_{\text{act}}(x)\le \epsilon$}
\State \Return $f_{\theta_0}(x; W_v)$  \Comment{Route to main memory}
\Else
\State $i^\star\gets \arg\max_i \Delta^{(i)}_{\text{act}}(x)$
\State \Return $f_{\theta_0}(x; W'_{v,i^\star})$  \Comment{Route to side memory $i^\star$}
\EndIf
\EndFunction
\end{algorithmic}
\end{algorithm}

\begin{algorithm}[t]
\caption{Training Subroutines}
\small
\begin{algorithmic}[1]

\Function{AssignShard}{$x$}
\State $a\gets \Call{FFNActivation}{x}$;\; $\Delta^{(i)}\gets \|a\cdot (W'_{v,i}-W_v)\|_2,\;i=1..K$
\State \Return $\arg\max_i \Delta^{(i)}$  \Comment{Use the most active shard during training}
\EndFunction

\Function{FormBatches}{$S,~b$}  \Comment{Distribution-aware batching}
%\State For each $x\in S$, compute feature $o(x)\gets \Call{Norm}{\Call{ModelFeat}{x}}$
\State $o_i \gets \mathrm{Norm}\!\big(\mathrm{ModelFeat}(x^{(i)})\big)$ for $i=0,\ldots,b-1$
\State Greedy seeding: pick $x^{(0)}=\arg\max_{x\in S}\frac{1}{|S|}\sum_{x'} \cos(o(x),o(x'))$
\State Build $B\gets \{x^{(0)}\}\cup\text{Top-(}b{-}1\text{)}$ nearest by cosine; remove $B$ from $S$
\State Repeat until $S$ is empty; \Return list of batches $\mathcal{B}$
\EndFunction

\Function{AutoregCE}{$B$} \Comment{Autoregressive edit loss $L_{\text{edit}}$}
\State $L\gets 0$
\For{$x\in B$ with target sequence $y$}
\State $L \gets L - \sum_{t=1}^{|y|}\log p_\theta(y_t\mid y_{<t},x)$
\EndFor \State \Return $L/|B|$
\EndFunction

\Function{IntraBatchKD}{$B,\lambda,\vartheta,T$}  \Comment{Eq.(4); optional soft-KD}
\State Compute $o_i \gets \mathrm{Norm}\!\big(\mathrm{ModelFeat}(x^{(i)})\big)$ for $i=0,\ldots,b-1$
\State $L_{\text{cos}}\gets \frac{1}{b-1}\sum_{i=1}^{b-1}\left(1-\frac{o_i^\top o_0}{\|o_i\|\|o_0\|}\right)$
\State $o_{\text{mean}}\gets \frac{1}{b}\sum_{i=0}^{b-1}o_i$;\; $L_{\text{var}}\gets \frac{1}{b}\sum_{i=0}^{b-1}\|o_i-o_{\text{mean}}\|_2^2$
\State $L\gets \lambda\,L_{\text{cos}}+\vartheta\,L_{\text{var}}$
\If{$T>0$} \Comment{Optional: KL distillation for added stability}
\State Get logits $z_i$; $p_i=\operatorname{softmax}(z_i/T)$;\; $L \gets L + \frac{1}{b-1}\sum_{i=1}^{b-1}\mathrm{KL}(p_0\|p_i)$
\EndIf
\State \Return $L$
\EndFunction

\Function{RoutingMargin}{$B,\gamma_1,\gamma_2,\gamma$} \Comment{Eq.(7)}
\State $L\gets 0$
\For{each edit sample $x_e\in B$}
\State sample unrelated $x_i$;\; compute $\Delta_e=\Call{ActDelta}{x_e}$, $\Delta_i=\Call{ActDelta}{x_i}$
\State $L\gets L+\max(0,\Delta_i-\gamma_1)+\max(0,\gamma_2-\Delta_e)+\max(0,\gamma-(\Delta_e-\Delta_i))$
\EndFor \State \Return $L/|B|$
\EndFunction

\Function{MaskedUpdate}{$W'_{v,i}, M_i, L$} \Comment{Masked gradient to reduce cross-shard interference}
\State $g\gets \nabla_{W'_{v,i}} L$; \; $g_{\text{m}}\gets M_i\odot g$  \Comment{$M_i\in\{0,1\}^{\text{shape}(W_v)}$}
\State $W'_{v,i}\gets \text{OptimizerStep}(W'_{v,i}, g_{\text{m}})$ \Comment{SGD/Adam, etc.}
\EndFunction

\Function{FilterAndRecluster}{$B,\epsilon_{\text{cons}},\mathcal{R}$}
\For{$x\in B$} \State $\ell_{\text{KD}}(x)\gets$ per-sample KD vs. $x^{(0)}$
\If{$\ell_{\text{KD}}(x)\ge \epsilon_{\text{cons}}$} \State move $x$ to $\mathcal{R}$ \EndIf
\EndFor \State \Return
\EndFunction

\Function{Evaluate}{$x_e,y_e$}
\State $\hat{y}\gets \Call{RouteAndPredict}{x_e}$;\; $c\gets \mathbf{1}[\hat{y}=y_e]$
\State \Return $(\hat{y},c)$
\EndFunction

\end{algorithmic}
\end{algorithm}

\begin{algorithm}[t]
\caption{Utility Functions}
\begin{algorithmic}[1]

\Function{ErrorRate}{$\mathcal{E},i$} \Comment{Error rate for shard $i$}
\State $\mathcal{E}_i\gets \{x\in \mathcal{E}\mid \arg\max_j \Delta^{(j)}_{\text{act}}(x)=i\}$
\State $r_i \gets \frac{|\{x\in \mathcal{E}_i \mid \Call{Correctness}{x}\le \tau_{\text{correct}}\}|}{|\mathcal{E}_i|}$
\State \Return $r_i$
\EndFunction

\Procedure{ReTrigger}{$\mathcal{E}$} \Comment{Closed-loop pruning and retraining}
\State $j\gets \arg\max_i \Call{ErrorRate}{\mathcal{E},i}$ \Comment{Identify worst-performing shard}
\State Remove or reinitialize shard $j$: $W'_{v,j}\gets W_v+\sigma_{\text{init}}\cdot\mathcal{N}(0,1)$; resample $M_j$
\State Build $\mathcal{E}_{\text{retrain}}$ from $\mathcal{E}$; form batches; retrain shards via \Call{MaskedUpdate}{} + \Call{IntraBatchKD}{}
\EndProcedure

\Function{LossAwareTIESMerge}{$\{W'_{v,i}\},W_v$} \Comment{Loss-aware weighted TIES merge}
\State For each shard $i$: $\tau_i\gets W'_{v,i}-W_v$;\; compute training loss $L_i$ on its assigned data
\State $w_i\gets \frac{e^{-\alpha L_i}}{\sum_j e^{-\alpha L_j}}$
\For{each parameter index $p$}
\State $S\gets \{(i,\tau_i[p],w_i)\}_{i=1}^K$
\If{all $\tau_i[p]$ share the same sign}
\State $\delta[p]\gets \sum_i w_i\,\tau_i[p]$ \Comment{Consistent signs: weighted sum}
\Else
\State $i^\star\gets \arg\max_i \{w_i\,|\tau_i[p]|\}$;\; $\delta[p]\gets \tau_{i^\star}[p]$ \Comment{Conflict: keep most trustworthy shard}
\EndIf
\EndFor
\State $W_v\gets W_v+\delta$;\; \Return $W_v$
\EndFunction

\Function{FFNActivation}{$x$} \State \Return activation $a(x)$ at the target FFN layer \EndFunction
\Function{ActDelta}{$x$} \State \Return $\max_i \|a(x)\cdot (W'_{v,i}-W_v)\|_2$ \EndFunction
\Function{ModelFeat}{$x$} \State \Return feature used for similarity (e.g., $a(x)$ or last-token state) \EndFunction
\Function{Norm}{$v$} \State \Return $v/\|v\|_2$ \EndFunction
\Function{Correctness}{$x$} \State \Return predicted correctness score for $x$ \EndFunction

\end{algorithmic}
\end{algorithm}
%%%%%%%%%%%%%%%%%%%%%%%%%%%%%%%%%%%%%%%%%%%%%%%%%%%%%%%%%%%%%%%%%%%%%%%%%%%%%%%
%%%%%%%%%%%%%%%%%%%%%%%%%%%%%%%%%%%%%%%%%%%%%%%%%%%%%%%%%%%%%%%%%%%%%%%%%%%%%%%

\end{document}